\documentclass{article}
\usepackage[compatibility=false]{caption}
\usepackage{jmlr2e}
\usepackage{graphicx,amsmath,booktabs,dsfont} 
\usepackage[utf8]{inputenc}
\usepackage{enumitem}
\usepackage{cite}
\usepackage{hyperref}
\usepackage{subcaption}
\usepackage{todonotes,float}

\usepackage{booktabs}
\numberwithin{equation}{section}

\renewcommand{\rho}{\varrho}

\newcommand{\R}{\mathbb{R}}

\newcommand{\Rc}{\mathcal{R}}
\newcommand{\Ec}{\mathcal{E}}
\newcommand{\Ac}{\mathcal{A}}
\newcommand{\gaussian}{\gamma}
\newcommand{\Gc}{\mathcal{G}}
\newcommand{\Sbb}{\mathbb{S}}

\newcommand{\yseq}{\{y_i\}_{i=1}^n}

\DeclareMathOperator*{\argmin}{arg\,min}

\newcommand{\TV}[1][]{
\ifthenelse { \equal {#1} {} }  
    {\operatorname{TV}} 
    {\operatorname{\ensuremath{#1}-TV}}
}

\newcommand{\Per}[1][]{
\ifthenelse { \equal {#1} {} }  
    {\operatorname{Per}} 
    {\operatorname{\ensuremath{#1}-Per}}
}

\newcommand{\esssup}[1][]{
\ifthenelse { \equal {#1} {} }  
    {\operatorname*{ess\,sup}} 
    {\operatorname*{\ensuremath{#1}-ess}\sup}
}
\newcommand{\essinf}[1][]{
\ifthenelse { \equal {#1} {} }  
    {\operatorname*{ess\,inf}} 
    {\operatorname*{\ensuremath{#1}-ess}\inf}
}
\newcommand{\essosc}[1][]{
\ifthenelse { \equal {#1} {} }  
    {\operatorname*{ess\,osc}} 
    {\operatorname*{\ensuremath{#1}-ess\,osc}}
}

\renewcommand{\subset}{\subseteq}

\graphicspath{{figures/}}

\newcommand{\off}[1]{{}}

\ShortHeadings{Critical Points of t-SNE}{Haridas and Murray}
\firstpageno{1}

\begin{document}

\title{On the Abundance of Critical Points of the t-SNE Energy}
\author{\name Nakul Haridas \email nakulh@ncsu.edu \\ \addr Department of Mathematics \\ North Carolina State University \\ Raleigh, North Carolina, USA \AND \name Ryan Murray \email rwmurray@ncsu.edu \\ \addr Department of Mathematics \\ North Carolina State University \\ Raleigh, North Carolina, USA }

\editor{}

\maketitle

\abstract{This paper considers the energy landscape of the t-SNE algorithm. While this algorithm has enjoyed broad adoption, the non-convexity of the associated energy has made it difficult to rigorously understand what the algorithm captures in many settings. In particular, a number of well-known numerical examples, several of which are reproduced in this article, suggest a complicated energy landscape with many local minimizers that do not respect the topology or clustering structure of the underlying data. This work seeks to provide first steps towards a rigorous explanation of these phenomena. Specifically, for a general family of energies, which include both the original t-SNE algorithm and recently identified large data limits, and for densities in feature space which obey a continuous symmetry, we construct infinite families of distinct critical points. These critical points are based upon identifying pairs of discrete symmetries, one in the original feature space and the other in the target embedding space, which are preserved under gradient dynamics. These critical configurations exhibit many characteristics, such as topology breaking and spurious clustering, which are often observed empirically. Finally, numerical and analytical examples are given throughout as a means of illustrating the approach.}


\section{Introduction}

High-dimensional data arises throughout science and engineering, and dimension reduction---representing such data in a few coordinates while preserving its essential structure---is a basic tool for both analysis and visualization. There are many methods designed for this task, most of which can be posed as optimization problems which seek to score the preservation of the data's underlying structure. While a few of these methods admit explicit minimizers, such as principal component analysis, many of the most popular methods do not admit explicit minimizers and rely upon gradient descent to find quality solutions.

 One such method is t-SNE \citep{JMLR:v9:vandermaaten08a}. This method seeks to preserve pairwise distances between data, with adaptive weights that prioritize shorter scales. Since the algorithm's introduction, t-SNE has seen widespread adoption across diverse scientific domains, with notable applications in neurology \citep{ridgway2012early}, genomics \citep{Kobak2019}, natural language processing \citep{Li2016}, and materials science \citep{kim2026unisensory}. 
 
 Because t-SNE utilizes a non-convex energy for constructing embeddings, solution formulas are not available and embeddings must be found via gradient descent. This lack of solution formulas makes it much more difficult to develop theory for t-SNE. In spite of this, over the last ten years there is a small, but growing, literature which seeks to develop theory for t-SNE, which we describe in Section \ref{sec:lit-review}.
 
 Our interest in this work lies in providing some first steps towards rigorously describing the complexity of the energy landscape for t-SNE. In particular, our goal is to demonstrate settings wherein we can prove that the t-SNE energy admits \emph{infinitely many distinct critical points}. The main structure of the work is as follows:
 
 \begin{enumerate}
 	\item We first give, in Section \ref{sec:energy_variations} a general formulation of the t-SNE energy, which includes several variants which have been introduced in the literature. Then in Section \ref{sec:Classic-Examples}, we reproduce a number of examples from previous works which highlight the complexity of the minimization problem, as well as various topological and geometric properties that arise in empirical settings.    
 	\item  In Section~\ref{sec:symmetries}, we develop a notion of a symmetry invariant couplings, which builds upon many classical definitions from the study of isometry groups and their actions. We give numerous examples throughout the section to aid readers who are unfamiliar with such symmetry groups.

   \item  In Section \ref{sec:descent-paths}, we prove the well posedness of the gradient flow \ref{eqn:GD-def}.
    It turns out that this gradient flow preserves symmetry invariance (Theorem~\ref{thm:sym-invar-preserved}). By utilizing a compactness argument, this immediately yields the existence of a symmetry invariant critical point for every
    symmetry pair (Theorem~\ref{thm:crit-exist}). In Section \ref{sec:num-examples}, we give numerical examples on the construction of such critical points on various manifolds.

    \item  In Section \ref{sec:distinctness}, we prove that in the Euclidean setting the trivial (i.e. constant) embedding  is not a local minimum of the t-SNE energy. For a wide class of symmetry pairs this is the only map which can simultaneously satisfy multiple symmetries simultaneously. Subsequently, we obtain the main result, (Theorem \ref{thm:abundant_critical_points}) which states that for radially symmetric input data there exist infinitely many distinct critical points on the Euclidean space for the t-SNE energy. Here distinctness is defined in a strong way, which precludes isometric transformations in the input or output spaces.
\end{enumerate}

The final main result relies upon two natural conjectures, which are the subject of current investigation. First, when $Y=\R^m$ the energy is only asymptotically
scale invariant, so that energy dissipation does not immediately control the spread of the
embedding, and the moment bound required for tightness as $t \to \infty$ remains
conjectural: recent work \citep{jeong2024convergenceanalysistsnegradient} proves this is the case in specific settings. Second, we do not resolve whether a flow started from a coupling induced by a map relaxes to a general coupling as $t \to \infty$: this hearkens to the theory of optimal transportation and has been recently proven for some simpler dimension reduction algorithms \citep{murray2025probabilistic}. We hope to address these issues in future work.

 \subsection{Related Literature}\label{sec:lit-review}

  Some early works focused on understanding the effect of  \emph{early-exaggeration}, a stage of the algorithm which emphasizes different parts of the energy at different stages of the optimization. \citet{linderman2017clusteringtsneprovably} prove that the early-exaggeration phase of t-SNE causes the diameter of each cluster's embedding to shrink, and show that in a suitable limit this phase reduces to a spectral clustering method.
\citet{arora2018analysis} strengthen this to a visualization guarantee, showing that under a deterministic separation condition on the data — satisfied by mixtures of well-separated log-concave distributions — early-exaggeration t-SNE produces a full visualization in which every cluster is visibly separated. \citep{cai2022theoreticalfoundationstsnevisualizing} established the asymptotic equivalence of the early exaggeration phase with power iterations in spectral embeddings.

The existence of a minimizer of the KL divergence in t-SNE in the finite data case was recently proved by \citep{jeong2024convergenceanalysistsnegradient}. In the continuum case, \citet{auffinger2023equilibriumdistributionstdistributedstochastic} identified a large data limit of the t-SNE energy over plans (couplings with fixed input marginal), assuming the perplexity grows linearly with the number of data points. \citet{murray2024largedatalimitsscaling} identified a different continuum limit, over maps assuming the perplexity grows faster than $\log n$ but slower than $n$, like $n^{\alpha}$ for $0< \alpha <1$. Notably, in their respective regimes \citet{auffinger2023equilibriumdistributionstdistributedstochastic} obtain a well-posed equilibrium over plans, whereas \citet{murray2024largedatalimitsscaling} prove an asymptotic in $n$ ill-posedness over maps and introduce a rescaling to recover a consistent limit.

\citet{tsne_no_clusters_yang} show empirically that t-SNE fails to reveal cluster structure on real-world data sets that are well-clusterable yet violate the conditions of \citet{arora2018analysis}, and that this failure persists across perplexity choices and alternative optimization algorithms.  Notably, they observe that a lower objective value function can correspond to a worse cluster visualization, a fact which we also observe in Section \ref{sec:Classic-Examples}.

The pathological behaviour of t-SNE has been observed empirically in  \citep{wattenberg2016use}, \citep{linderman2017clusteringtsneprovably}. The effects of initialization on the t-SNE output has been discussed in \citep{kobak2021initialization}

In the very recent paper \citep{li2026theoreticallimitationstsne}, the authors
prove a symmetry-driven failure mode of t-SNE. Their starting point is that when
the input points are equidistant, the configuration is symmetric under permuting
the points, which forces each conditional affinity $P_i$ to be uniform no matter
how the bandwidths are chosen, so the input distribution $P$ is uniform. For a large number
of points drawn from the hypersphere $\mathbb{S}^{d-1}$, the global minimizer will have most of its mass concentrated on a small
ball in the embedded space. Put in other words, the only way to minimize the t-SNE
energy of a large number of (approximately) equidistant points in high dimension is
to map all the points to a small neighborhood.

One context where the consequences of symmetry invariance of the t-SNE energy has been very recently investigated is \citep{bergam2025tsneexaggeratesclustersprovably}. In that paper they use a type of scale and translation symmetry to argue that any $n$-point embedding which is a critical point of the t-SNE energy can be transformed to an $n$-point configuration which is arbitrarily close to the $n$-simplex and which is also, in an appropriate sense, a critical point. The authors use this to then argue a type of instability of the embeddings learned by t-SNE. While our work is related in that we seek to utilize symmetries to study energy landscapes, it is different in the sense that we are using a broader class of \emph{discrete} symmetries in order to quantify the richness of the family of critical points.

\section{Energetic models}
\subsection{Review of original t-SNE Formulation}
t-SNE is a dimension reduction technique used to learn high dimensional structure in low dimensional space  \citep{JMLR:v9:vandermaaten08a}. We consider a family of points $\{x_i\}_{i=1}^n$ with $x_i \in \R^d$ representing our original data set in high dimensions. We will generally assume that the $x_i$ are i.i.d. sampled from an underlying distribution on $\R^d$ with density $\rho$.The objective of t-SNE is to construct an associated family of points $y_i$ in $\mathbb{R}^{m}$, typically with  $m \ll d$, which approximately preserves the structure of the original points $x_i$. In the language of geometry, one can think of the $y_i$ as an \emph{embedding} of the high-dimensional $x_i$ into a lower-dimensional space. From a high level, the t-SNE algorithm seeks to find such an embedding which locally preserves pairwise distances.\\

We now describe the specifics of the energy which t-SNE uses to quantify the quality of an embedding. Initially, we calculate an affinity matrix in the high-dimensional space in the following way: For points $i \neq j$, we define 
\[
p_{j|i}
=
\frac{
\exp\!\left(-\frac{\lvert x_i - x_j\rvert^2}{2\sigma_i^2}\right)
}{
\sum\limits_{k \neq i}
\exp\!\left(-\frac{\lvert x_i - x_k\rvert^2}{2\sigma_i^2}\right)
},
\quad j \neq i,
\qquad p_{i|i} := 0,
\]
\[
p_{ij} = \frac{p_{j|i} + p_{i|j}}{2n}, \quad i \neq j,
\qquad p_{ii} := 0.
\]
We use $|\cdot |$ to denote the norm in Euclidean spaces, and all logarithms are natural. We define $P_{\cdot|i} = (p_{j|i})_{j \neq i}$ as the $i$th conditional distribution, which encodes the similarity of the $i$th data point with other points; it is a probability vector, whereas the $i$th row of the symmetrized matrix $P = (p_{ij})$ is not.
We note that the $p_{ij}$ are symmetrized and normalized such that $\sum_{i \neq j} p_{ij} = 1$. Here $\sigma_i$ is a locally-adaptive bandwidth associated with $x_i$. The value of $\sigma_i$ is determined implicitly by the constraint
\[
H_i(\sigma_i) := -\sum_{j \neq i} p_{j|i} \ln (p_{j|i}) = \ln \kappa .
\]
Here $\kappa$ is a hyperparameter known as the \emph{perplexity}.  Changing the perplexity changes the matrix $P$, and hence the objective function below.
$H_{i}(\sigma_i)$ is the information entropy of point $i$ measured in nats, namely
\[
\operatorname{Perp}(P_{\cdot|i}) \;=\; e^{H_i(\sigma_i)}
\;=\; \prod_{j \neq i} \left( \frac{1}{p_{j|i}} \right)^{p_{j|i}}
\;=\; \kappa .
\]
The $p_{ij}$ can be interpreted as a smoothed version of $\kappa$ nearest neighbors. A higher value of $p_{ij}$ indicates that the points $x_i,x_j$ are more similar, relative to the local scales $\sigma_i, \sigma_j$. The $\sigma_i$ is a ``locally adaptive bandwidth,'' chosen so that each point sees approximately the desired number of nearest neighbors

The similarity between the embedded points $\{y_i\}$ is measured using the heavy-tailed Student's t-distribution (packaged into the matrix $Q$),
\[
q_{ij} = \frac{\left(1 + |y_i - y_j|^2\right)^{-1}}
              {\sum_{k \neq l} \left(1 + |y_k - y_l|^2\right)^{-1}}, 
\quad i \neq j.
\]
Here the affinities are measured using a heavy-tailed distribution because it promotes greater dispersion of the embedded data, and empirically biases the embeddings away from collapsing to a single mode \citep{JMLR:v9:vandermaaten08a}.

The t-SNE algorithm minimizes the Kullback-Leibler divergence between the probability matrices corresponding to the original and embedded data points, namely
\[
\mathrm{KL}(P \,\|\, Q) 
= \sum_{i \neq j} p_{ij} \, \log \frac{p_{ij}}{q_{ij}}.
\]
This function is differentiable with respect to $y_i$, and often algorithms  use gradient descent on the entire family of embedded points $\{y_i\}_{i=1}^n$ to numerically minimize this energy.
\\\\
We can expand the Kullback--Leibler divergence as
\[
\mathrm{KL}(P \,\|\, Q) = \sum_{i \neq j}p_{ij} \log(p_{ij})-\sum_{i \neq j}p_{ij}\log(q_{ij}) .
\]
Having defined the bandwidth, the problem of minimizing the Kullback Liebler divergence is equivalent to minimizing $-\sum_{i \neq j}p_{ij}\log(q_{ij})$, as the first term is purely data dependent, that is, it is fixed by the input data $\{x_i\}$. 
\\

In the interest of minimizing this energy over $\yseq$, it is helpful to rewrite it in a different form. Indeed, the energy may be written, after removing terms independent of $\yseq$, as the sum of attraction and repulsion energies, namely
\[
\mathrm{A_n}[\{y_i\}_{i=1}^{n}]:=\frac{1}{n} \sum_{i=1}^{n}\frac{\sum\limits_{j\neq i} \exp\left(-\frac{|x_i-x_j|^2}{2\sigma_i^2}\right)\log(1+|y_i-y_j|^2)}
{\sum\limits_{j \neq i} \exp\left(-\frac{|x_i-x_j|^2}{2\sigma_i^2}\right)},
\]
\[
\mathrm{R_n}[\yseq]=:\log \bigg(\frac{1}{n^2}\sum_{j \neq i} \frac{1}{1+|y_i-y_j|^2}\bigg),
\]

\[
\argmin_{\yseq} \; \mathrm{KL}[\yseq] \;=\; \argmin_{\yseq}\bigg(\mathrm{A_n}[\yseq] + \mathrm{R_n}[\yseq] \bigg).
\]

Here, $\mathrm{A_n}$ can be interpreted as the attraction energy and $\mathrm{R_n}$ as the repulsion energy. The function $\mathrm{A_n}$ is nonnegative and vanishes precisely when all of the $y_i$ coincide, so it is minimized when all high-dimensional points are mapped to a single point; $\mathrm{R_n}$ is strictly decreasing in each pairwise distance, and its infimum is approached, but not attained, as the mutual distances diverge.

\subsection{Continuum Model}
The form of the attraction and repulsion energies naturally suggests a large data limit as $n \to \infty$. This limit has been considered previously, under different assumptions, in both \citep{auffinger2023equilibriumdistributionstdistributedstochastic} and \citep{murray2024largedatalimitsscaling}. In both works a significant challenge is the implicitly-defined bandwidth $\sigma_i$. 

In order to obtain a deterministic limit for $\sigma_i$, both of these works choose to allow the perplexity $\kappa$ to grow as a function of $n$. For example, in \citep{auffinger2023equilibriumdistributionstdistributedstochastic} the authors choose to let $\kappa = \delta n$, for some $\delta \in (0,1)$, and they relax the embedding problem to one over couplings, $\pi$, between $\R^d \times \R^m$, that is probability measures on the product space. They then prove that the bandwidths $\sigma_i$ converge, as $n \to \infty$ towards a deterministic function $\sigma^*(x)$, which depends upon $\delta$ and $\pi$, and satisfies an integral equation. In turn, this implies that one candidate large-data limit for t-SNE is given by 
\[
\mathcal{B_{\delta}}(\pi) = \iint 
p_{\sigma^{*}}(x, x') \,
\log\!\left( \frac{p_{\sigma^{*}}(x, x')}{q(y, y')} \right)
\, \pi(dx, dy) \, \pi(dx', dy'),
\]
where $p$ is a (non-locally) normalized Gaussian measure of similarity depending on $\sigma^{*}$,  $q$ is  given by the normalized algebraic similarities, and $\pi$ is a joint distribution over $(x,y)$ pairs. That work studies the existence of minimizers of $\mathcal{B_{\delta}}$ and the convergence of the minimizers of the discrete $\mathrm{KL}$ energy towards the minimizers of $\mathcal{B_{\delta}}$ in an appropriate sense. These results are attained at the cost of requiring very large perplexities: much larger than many practical settings.

On the other hand, in \citep{murray2024largedatalimitsscaling} the authors consider perplexities that grow much more slowly (e.g. $\sim n^\alpha$ for any $0< \alpha<1$), under the setting where the input distribution admits a density $\rho$. They then show the existence of an explicit function $\sigma(x)$ (depending only on the density at $x$) and an explicit sequence $h_n \to 0$ so that $\sigma(x_i) -h_n \sigma(x) \to 0$. The authors then propose studying the energies
\begin{align}\label{eqn:non-local-Pickarski}
{A}_h[T] &:= 
\int_{\R^d} 
\frac{
\int_{\R^d} 
\exp\!\left( -\frac{|x - x'|^{2}}{2h^2\sigma^{2}_{}(x)} \right) 
\log\!\big(1 + |T(x)-T(x')|^2\big)\,\rho(x')\,dx'
}{
\int_{\R^d} 
\exp\!\left( -\frac{|x - x'|^{2}}{2h^2\sigma^{2}_{}(x)} \right) 
\rho(x')\,dx'
}
\,\rho(x)\,dx, \\
R[T] &:= 
\log\!\left(
\iint_{\R^d \times \R^d} 
\frac{1}{1 + |T(x)-T(x')|^2
} \,
\rho(x)\,\rho(x')\,dx\,dx'
\right)
\\
\inf_{T} \; \mathcal{C}(T) &:=\; \inf_{T} \Big( A_h[T] + R[T] \Big) .\end{align}
In order to determine candidate large data limits of t-SNE, it is then natural to consider properties of this energy as $h \to 0$, for fixed embedding functions $T$.

This type of energy has several advantages from the analytical perspective. First, it has  replaced a discrete stochastic energy with a deterministic one. This is achieved by ignoring various stochastic effects: by replacing all sums over $i,j$ with integrals and replacing the $\sigma_i$ with the explicit $h \sigma(x)$. This, of course, does not directly prove anything about the limits of the original t-SNE energy as the derivation has, without justification, interchanged a number of limits. We consider the problem of rigorously proving limits of minimizers of the t-SNE energy as $n \to \infty$ and for realistic perplexities to be a challenging, open problem, which we do not seek to address in this work.

Second, while the energy is non-local, in the sense that there are relationships between $x$ and $x'$ inside the integrals, these relationships are largely localized and approximately sparse. In \citep{murray2024largedatalimitsscaling} they identify a possible rescaling of $A_h$ which leads to a well-posed localized limit. Similarly, \citep{calder2026continuumlimittsnedata} conducts a rescaling of the embedded space in order to identify a candidate large data-limit for the original energy.

\subsection{Generalized Model}\label{sec:energy_variations}

We now introduce a generalized version of the attraction and repulsion energies which includes i) the original discrete energy from \citep{JMLR:v9:vandermaaten08a}, ii) the weakly non-local energy from \citep{murray2024largedatalimitsscaling}, and iii) the strongly non-local energy from \citep{auffinger2023equilibriumdistributionstdistributedstochastic}. The  formulation that we give here is also applicable in contexts where the feature space and embedding spaces are given by manifolds. The manifold generalization is useful for some of the examples we give, and is also relevant in more contemporary algorithms such as contrastive learning; see for example \citep{chen2020simple} and \citep{wang2020understanding}.

We consider two Riemannian manifolds $X, Y$ with associated metrics $d_X$ and $d_Y$. We call $X$ the \emph{feature space} and $Y$ the \emph{embedding space}. We assume that on both spaces we have  functions
\[
\eta_X : X \times X \to [0,\infty), \qquad \eta_Y : Y \times Y \to [0,\infty).
\]
One should think of $\eta$ as a function which is smooth, symmetric, nonnegative, vanishes exactly on the diagonal, and is comparable to the squared distance: there exist constants $0<c\le C$ such that
\[
c\,d(a,b)^2 \le \eta(a,b) \le C\,d(a,b)^2.
\]
We give the precise assumptions upon $\eta$ below, but give some examples here.

\begin{example}[Standard Euclidean Case]
    In the original algorithm, $X = \R^d$, $Y=\R^m$, and $\eta_X(x_1,x_2) = |x_1-x_2|^2$, $\eta_Y(y_1,y_2) = |y_1-y_2|^2$. We call this the \emph{standard Euclidean case} throughout the paper.
\end{example}

\begin{example}[Torus]
     For the flat torus $X = \mathbb{T}^d=[0,1)^{d}$ one option is to set 
     \[
     \eta_X(x_1,x_2) = \xi(d_X(x_1,x_2)^2),
     \]
     where $d_X(x_1, x_2)$ is the geodesic distance on $\mathbb{T}^d$. Here we need to select $\xi$ to have the right properties to imply smoothness of $\eta_X$ on all of $\mathbb{T}^d$: one such choice would be to let $\xi_0:\R^+ \to [0,1]$ be a $C^\infty$ function which takes the value $1$ for $t < 1/3$ and $0$ for $t>2/3$, and then let $\xi(t) = \int_0^t \xi_0(s)\,ds$.
\end{example}

\begin{example}[Unit Sphere]
    In the case where $X = \mathbb{S}^{d-1}$ a natural choice is  $\eta(x_1,x_2) = \alpha(1-\cos(d_X(x_1,x_2)))$. 
\end{example}
\begin{example}[Hyperbolic Space; Poincare Ball Model]
    We consider $d$-dimensional hyperbolic space $X = \mathbb{H}^d$ in the Poincare ball model, with the associated metric $\eta(x_1,x_2) = d_X(x_1,x_2)^2$. This space is commonly parametrized by writing
    \[
    d_X(x_1,x_2)
    = \operatorname{arcosh}\!\left(
    1 + \frac{2|x_1 - x_2|^2}{(1-|x_1|^2)(1-|x_2|^2)}
    \right).
    \]
    
\end{example}

We will let $\mu_X$ denote the measure of the data on the manifold $X$: in the case from the previous section this $\mu_X$ would either have the density $\rho$ or would be an empirical measure, namely $\frac{1}{n} \sum_{i=1}^n \delta_{x_i}(dx)$, where $x_i$ are i.i.d from $\rho$.

Let $\Pi(\mu_X)$ denote the set of all couplings between $X,Y$, namely probability measures on $X \times Y$, which have marginal $\mu_X$ distribution in the feature space $X$. This space is also known as \emph{fibred probability measures}. These couplings are also sometimes called ``plans'' in the optimal transportation community, to distinguish them from maps $T:X \to Y$. The relaxation from maps to couplings is now a standard approach in the context of optimal transportation, and stems from the fact that the energies we consider do not immediately impose any single-valuedness.

Given a probability measure $\pi \in \Pi(\mu_X)$ and a continuous function $\sigma:X \to \R$, we can define the generalized energy as 

\begin{align}\label{eqn:generalized}
\mathcal{A}(\pi) &:= 
\int_{X \times Y} 
\frac{
\int_{X \times Y} 
\exp\!\left( -\frac{\eta_X(x,x')}{2\sigma^{2}_{}(x)} \right) 
\log\!\big(1 + \eta_Y(y,y')\big)\,\pi(dx'dy')
}{
\int_{X} 
\exp\!\left( -\frac{\eta_X(x,x')}{2\sigma^{2}_{}(x)} \right) 
\mu_X(dx')
}
\,\pi(dxdy), \\
\mathcal{R}(\pi) &:= 
\log\!\left(
\int_{X \times Y}\int_{X \times Y} 
\frac{1}{1 + \eta_Y(y,y'
)} \,
\pi(dxdy) \pi(dx'dy')
\right)
,\end{align}
\begin{equation} \label{eqn:total-energy}
\mathcal{C}(\pi) \; := \; \Big( \mathcal{A}(\pi) + \mathcal{R}(\pi) \Big), 
\end{equation} 
and we seek to minimize $\mathcal{C}$ over $\Pi(\mu_X)$.
For notational convenience, we let
\[
\gaussian (x,x')=\frac{\exp\big(\frac{-\eta_X(x,x')}{2\sigma^2(x)}\big)}{\int_X \exp\big(\frac{-\eta_X(x,x')}{2\sigma^2(x)}\big) \mu_X(dx')}
\]
We remark that $\gamma$ is not generally symmetric in $x,x'$, and the first argument is the value at which we are evaluating $\sigma$.


We notice if $\pi(dxdy)=\delta_{T(x)}(dy)\mu_X(dx)$, for some map $T$ from $X$ to $Y$, $X=\R^d$, $Y = \R^m$, and $\mu_X(dx) = \rho(x)dx$ this recovers the non-local energy from \citep{murray2024largedatalimitsscaling}. In the Euclidean setting when $\pi(dxdy)=\frac{1}{n}\sum_{i =1}^{n}\delta_{(x_i,y_i)}(dxdy)$ we recover the original t-SNE energy (after making an appropriate choice of $\sigma$ based upon the empirical sample $\{x_i\}_{i=1}^n$). \footnote{There is a minor difference between our setting and the original algorithm, as we do not exclude the $j=i$ term and hence the denominator of the attraction energy becomes $\sum_{j\neq i}\exp\left(\frac{-|x_i-x_j|^2}{2\sigma(x_i)^2}\right) + 1$ as opposed to $\sum_{j\neq i}\exp\left(\frac{-|x_i-x_j|^2}{2\sigma(x_i)^2}\right)$} Finally, by solving an implicit equation for $\sigma(x)$, we obtain the energy from \citep{auffinger2023equilibriumdistributionstdistributedstochastic}.



Central to the results of this paper is the first variation of this generalized energy. In particular, we will consider how the energy varies along vector fields that smoothly change $\pi$ in $y$: this is consistent with the particle-based methods utilized for optimizing this energy in empirical settings. It is also consistent with the formal Riemannian structure associated with Wasserstein spaces in optimal transportation, with the natural modification to the product space $X\times Y$ with an $X-$marginal constraint, that is to the fibered setting.

More precisely, we consider a function $\phi_t:X \times Y \to Y$ that is $C^1$ in $t$, and that satisfies $\phi_0(x,y) = y$. We then define a flow map on the product space via
\[
\Phi_t(x,y)=
(x ,
\phi_t(x,y))
\]
Under this flow map, we define $\pi_t=\pi_0 \circ\Phi_t^{-1}$, where $\Phi_t^{-1}$ is defined to mean $\pi_t(A \times B) = \pi(\{(x,y): x\in A,\phi_t(x,y) \in  B\})$. Alternatively, one can write this in terms of the pushforward measure, namely $\pi_t = (\Phi_t)_\sharp \pi$.

In order to appropriately define necessary conditions and gradient descents, it is useful to consider the vector field $\tau(x,y) = \frac{d}{dt}\phi_t(x,y) |_{t=0}$. One can show that for any two flow maps with the same initial vector field $\tau$, the time derivative of the energies at $t=0$ will be the same. One straightforward flow map generated by a particular choice of $\tau$ is given by $\phi_t(x,y)=\exp_{y}(t\tau(x,y))$, where $\tau(x,y)$ belongs to $Tan_y Y$. In the Euclidean case, this becomes $\phi_t(x,y) = y + t \tau(x,y)$.
Using this choice of representative, we then can compute the first variation
\[
\delta \mathcal{R}[\pi,\tau]:=\frac{d}{dt} \mathcal{R}(\pi_t) |_{t=0}=\int_{X \times Y} \tau(x,y) \mathcal{K}_{\mathcal{R}}(x,y) \pi(dxdy),
\]

\begin{equation}\label{eqn:repulsive_kernel}
    \mathcal K_{\Rc}[\pi](x,y)=\frac{-2}{I(\pi)}\int_{X \times Y}\frac{\partial_{1}\eta_Y(y,y')}{(1+\eta_Y(y,y'))^2 } \pi(dx'dy'),
\end{equation}

where
\[
I(\pi)=\int_{X \times Y}\int_{X \times Y} \frac{1}{1+\eta_Y(y,y') }\pi(dxdy)\pi(dx'dy').
\]
Similarly,

\[
\delta \mathcal{A}[\pi,\tau]=:\frac{d}{dt}\mathcal{A}(\pi_t) |_{t=0}=\int_{X \times Y} \tau(x,y) \mathcal{K}_{\mathcal{A}}(x,y) \pi(dxdy),
\]
where,
\begin{equation}\label{eqn:repulsive_kernel}
\mathcal{K}_{\mathcal{A}}[\pi](x,y)=2\int_{X \times Y} \Xi(x,x')\frac{\partial_{1}\eta_Y(y,y')}{(1+\eta_Y(y,y')) } \pi(dx'dy'),    
\end{equation}

and
\begin{equation} \label{eqn:normalizing_xi}
    \Xi(x,x')=\frac{\gaussian(x,x')+\gaussian(x',x)}{2}.
\end{equation}

\begin{remark}
    In the standard Euclidean case, we have
    \[
\mathcal{K}_{\mathcal{A}}[\pi](x,y)=4\int_{X \times Y} \Xi(x,x')\frac{y-y'}{(1+|y-y'|^2) } \pi(dx'dy') 
\]
and 
\[
\mathcal K_{\Rc}[\pi](x,y)=\frac{-4}{I(\pi)}\int_{X \times Y}\frac{y-y'}{(1+|y-y'|^2)^2 } \pi(dx'dy').
\]
\end{remark}

We then define
\begin{align}
     \Ec[\pi](x,y) &= \mathcal{K}_{\mathcal{A}}[\pi](x,y) + \mathcal K_{\Rc}[\pi](x,y)\\
     &=\int_{X \times Y} \Bigg[\Xi(x,x')\frac{\partial_{1}\eta_Y(y,y')}{(1+\eta_Y(y,y')) } - \frac{1}{I(\pi)} \frac{\partial_{1}\eta_Y(y,y')}{(1+\eta_Y(y,y'))^2 }  \pi(dx'dy')\Bigg] =:\int_{X \times Y} F_{\pi}(x,y,x',y')\pi(dx',dy').
\end{align}
We note that $\Ec[\pi](x,y)$ lies in the tangent space of $Y$. We denote the tangent space of $Y$, as $Tan_Y$.

One can directly show that a necessary condition of optimality for $\mathcal{C}$ is that
\[
 \Ec[\pi](x,y) = 0
\]
for $\pi$ almost every $x,y$. This is a generalization of the necessary condition given in the finite particle case in \citep{JMLR:v9:vandermaaten08a}, which identifies the necessary condition as a balance between attraction and repulsion forces, namely a balance between $\mathcal{K}_{\Ac}$ and $\mathcal{K}_{\Rc}$.

One formally defines the \emph{gradient flow} of the energy $\Ec$ via the system of equations
\begin{equation}\label{eqn:GD-def}
\begin{aligned}
\frac{\partial \phi_t(x,y)}{\partial t}&=-(\mathcal{K}_{\mathcal{A}}[\pi_t](x,\phi(x,y,t))+\mathcal{K}_{\mathcal{R}}[\pi_t](x,\phi(x,y,t)))\\
\pi_t&=\pi \circ\Phi_t^{-1}=\Phi_t \sharp \pi \\
\Phi_t(x,y)&=
(x ,
\phi_t(x,y))\\
\phi(x,y,0)&=y.
\end{aligned}
\end{equation}
We note that the coupled nature of $\phi_t$ and $\pi_t$ means that one needs to prove that this gradient flow is well-defined, as standard results don't immediately apply. However, classical tools can easily be adapted to prove well-posedness: we delay the proof until Section \ref{sec:descent-paths}, see Proposition \ref{prop:soln-exists}.


We note that this gradient flow can formally be written via the continuity equation as
\[
\partial_t\pi_t + \nabla_{y}.(\tau_t\pi_t)=0
\]
where $\tau_t = - (\mathcal{K}_{\mathcal{A}}[\pi_t]+\mathcal{K}_{\mathcal{R}}[\pi_t])$, and where the continuity equation has to be interpreted in a weak sense if $\pi$ does not have a smooth density in $y$. 

\begin{remark}
The field $\Ec[\pi]$ should be interpreted as the gradient of the energy
$\mathcal{C}$ at $\pi$.The energy therefore decreases the most along $\tau =
-\Ec[\pi]$, which is the sense in which \eqref{eqn:GD-def} is a gradient flow,
and along its solutions
\[
  \frac{d}{dt}\,\mathcal{C}(\pi_t) \;=\; -\int_{X\times Y}\big|\Ec[\pi_t](x,y)\big|^2\,\pi_t(dx\,dy) \;\le\; 0,
\]
with equality precisely when $\Ec[\pi_t] = 0$ holds $\pi_t$-almost everywhere.
\end{remark}

\subsection{Assumptions}
Here, we state the assumptions which are used in the proofs downstream. The proofs cite the assumptions from here as required. We mention that all the assumptions in this paper will be denoted with Roman numerals.
\begin{enumerate}[label=(\Roman*)]
\item $\eta_Y(a,b) = \xi(d_Y^2(a,b))$. \label{assump-symmetric_eta}

\item $\eta_Y \in C^2$ with $D^2\eta_Y$ and $\frac{\partial_1 \eta_Y}{1+\eta_Y}$ both in $ L^\infty(Y)$. \label{assump-eta_lipshitz}



\item $\sigma : X \to (0,\infty)$ is continuous with
      $0 < \sigma_{\min} \le \sigma $. \label{assump-Xi}
\end{enumerate}

We note that in the Euclidean case with $\eta$ being the norm squared assumptions \ref{assump-symmetric_eta} and  \ref{assump-eta_lipshitz} are naturally satisfied.

\section{Energy Landscape for t-SNE}

\subsection{Numerical illustrations of complexity of energy landscape}\label{sec:Classic-Examples}

 Here we review some of the previous results regarding different types of behavior associated with local minima of this problem. As the t-SNE energy is non-convex, there are very few settings where any explicit minimizers are known (limited results are known in some clustered settings \citep{linderman2017clusteringtsneprovably} and sparse regimes \citep{li2026theoreticallimitationstsne}). Accordingly, much of the work exploring properties of minimizers has been empirically-focused in settings involving low dimensional manifolds. We independently reproduce the following prominent examples:


\begin{example}[\emph{``Spaghetti Plots''} \citep{linderman2017clusteringtsneprovably}]
Here, following \citep{linderman2017clusteringtsneprovably}, we consider $n$ data points which are uniformly distributed along a straight line for $d=3$. This data is then embedded, via numerical minimization of the t-SNE energy, into $m=2$ dimensions. These are sometimes colloquially called ``spaghetti plots''. We find approximate optimizers using the openTSNE package \citep{JSSv109i03} with $N=2000$ and perplexity=5. Figure \ref{fig:spaghetti} shows a) the original data set, b) the output of t-SNE under random initialization, c) the output of t-SNE using PCA initialization, d) an analytical baseline given by linearly projecting the data to two dimensions without optimizing the energy. For reference, we list the energy values associated with these three embeddings in Table \ref{tab:spaghetti}. 

\vspace{0.5 cm}

\begin{figure}[h!]
    \centering

    \begin{subfigure}{0.23\linewidth}
        \centering
        \includegraphics[width=\linewidth]{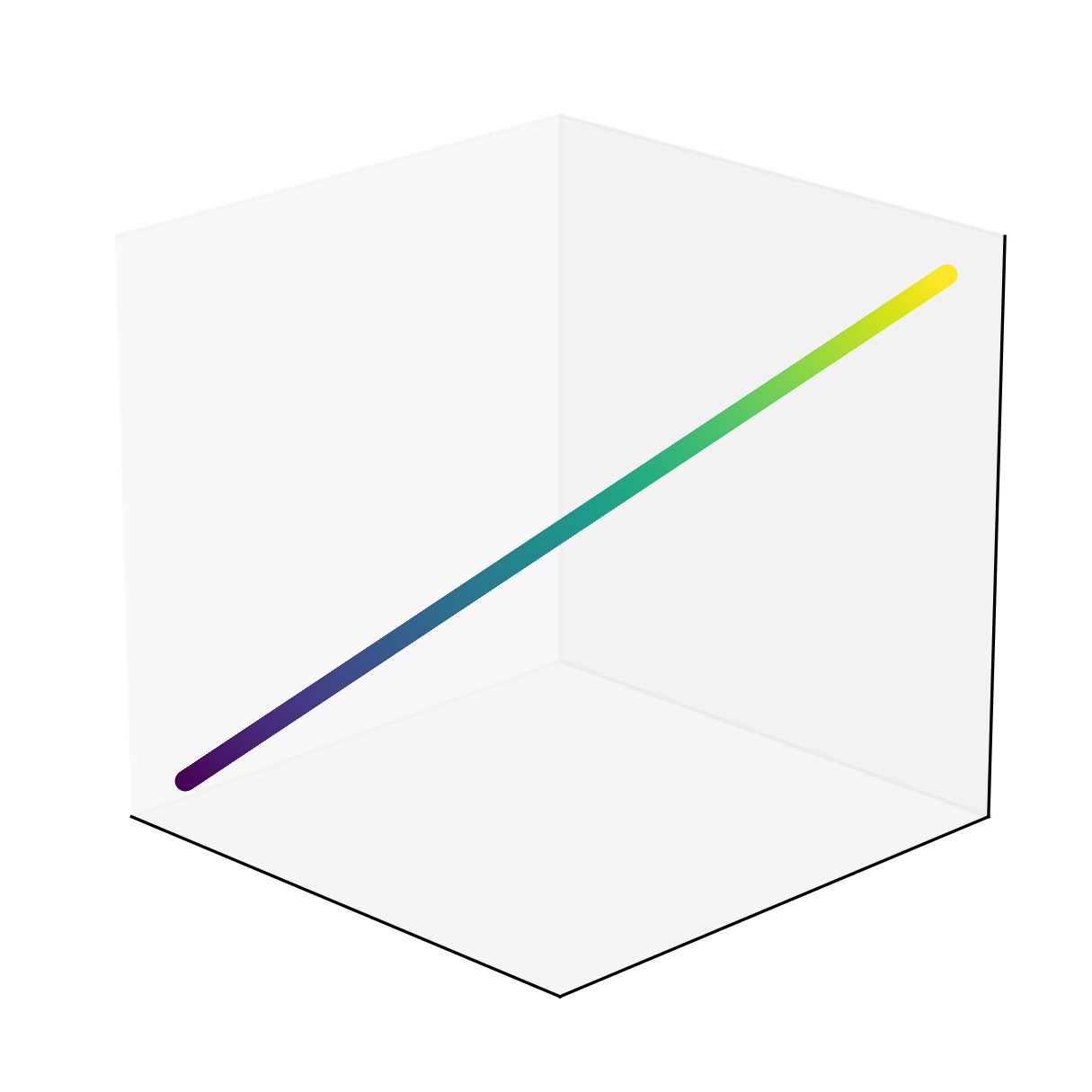}
        \caption{3-D line}
    \end{subfigure}
    \hfill
    \begin{subfigure}{0.23\linewidth}
        \centering
        \includegraphics[width=\linewidth]{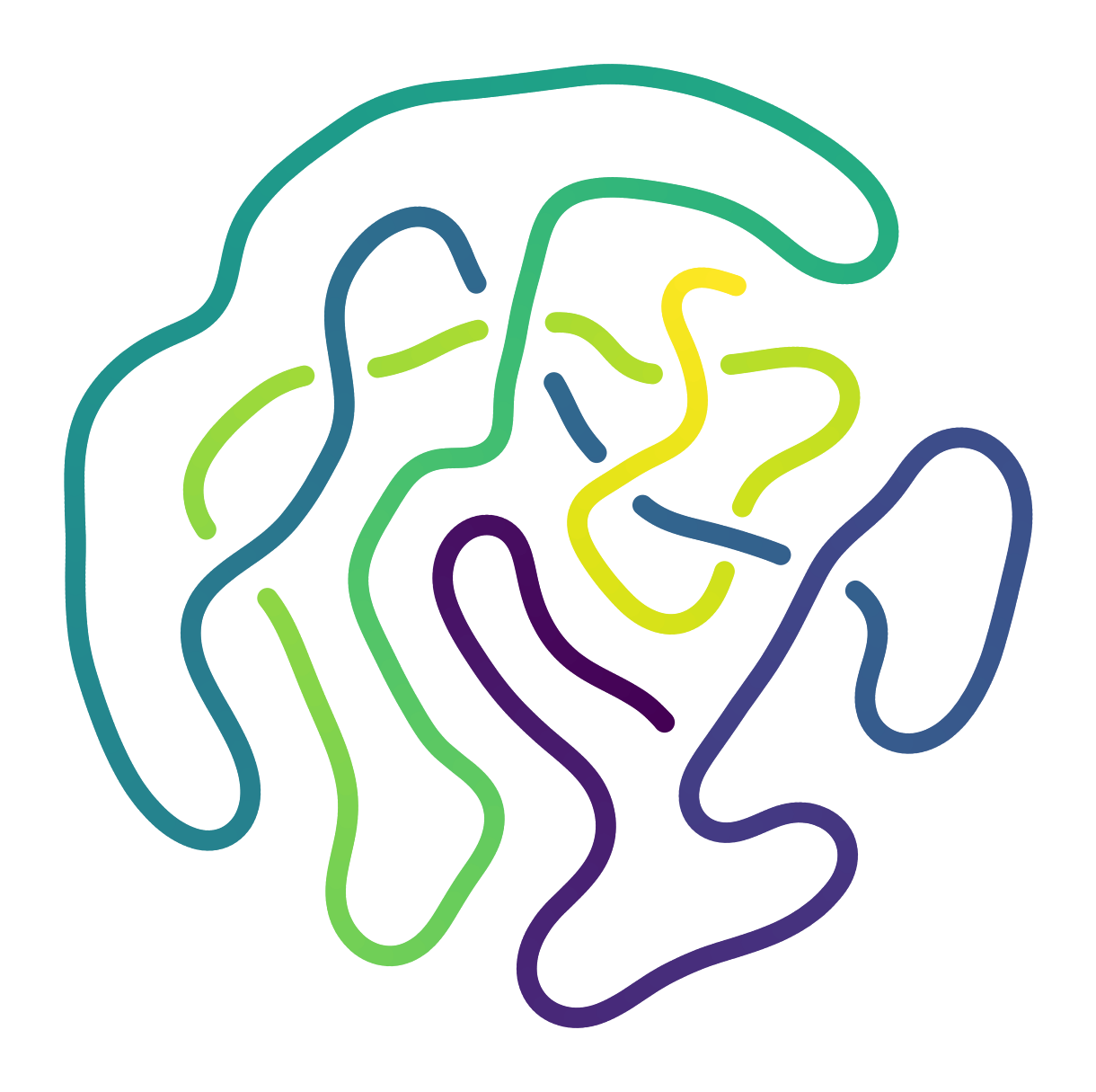}
        \caption{Random }
    \end{subfigure}
    \hfill
    \begin{subfigure}{0.23\linewidth}
        \centering
        \includegraphics[width=\linewidth]{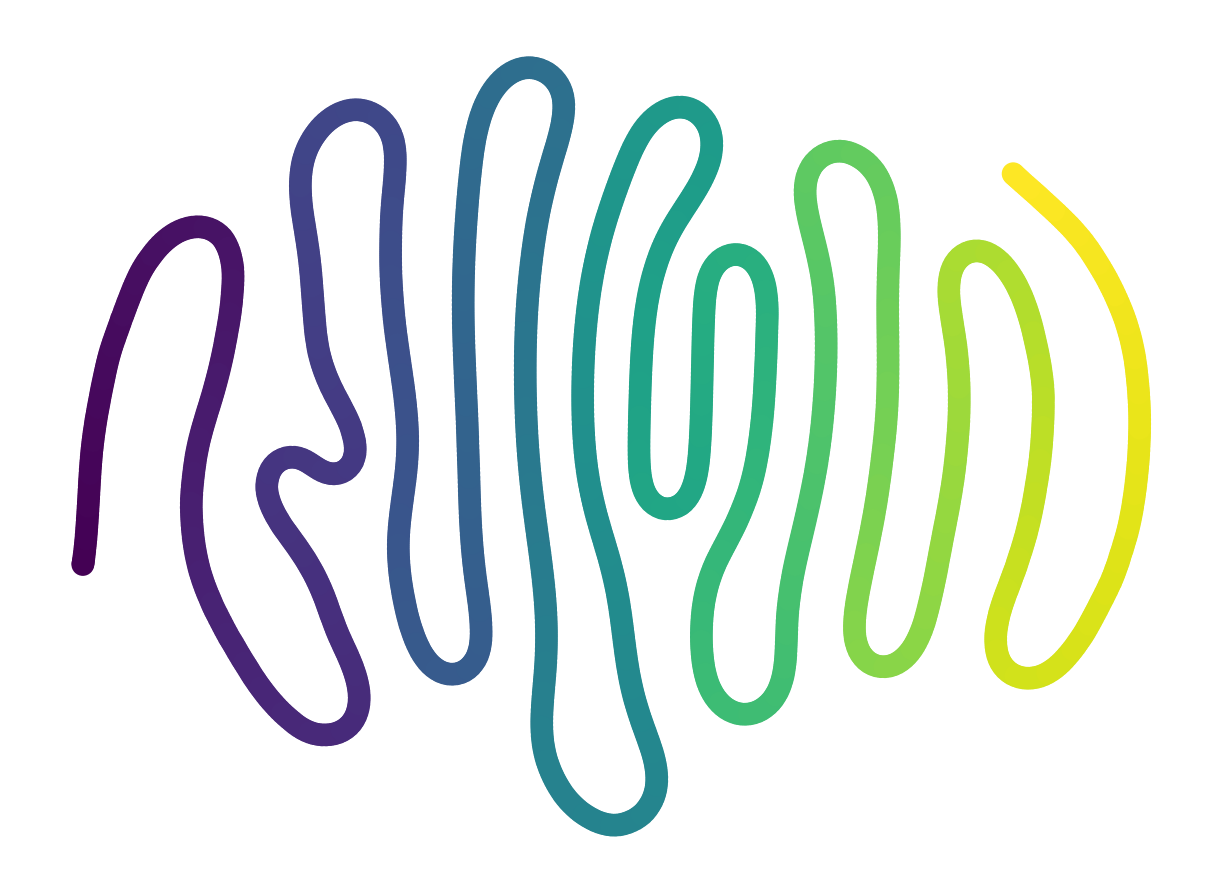}
        \caption{PCA }
    \end{subfigure}
    \hfill
    \begin{subfigure}{0.23\linewidth}
        \centering
        \includegraphics[width=\linewidth]{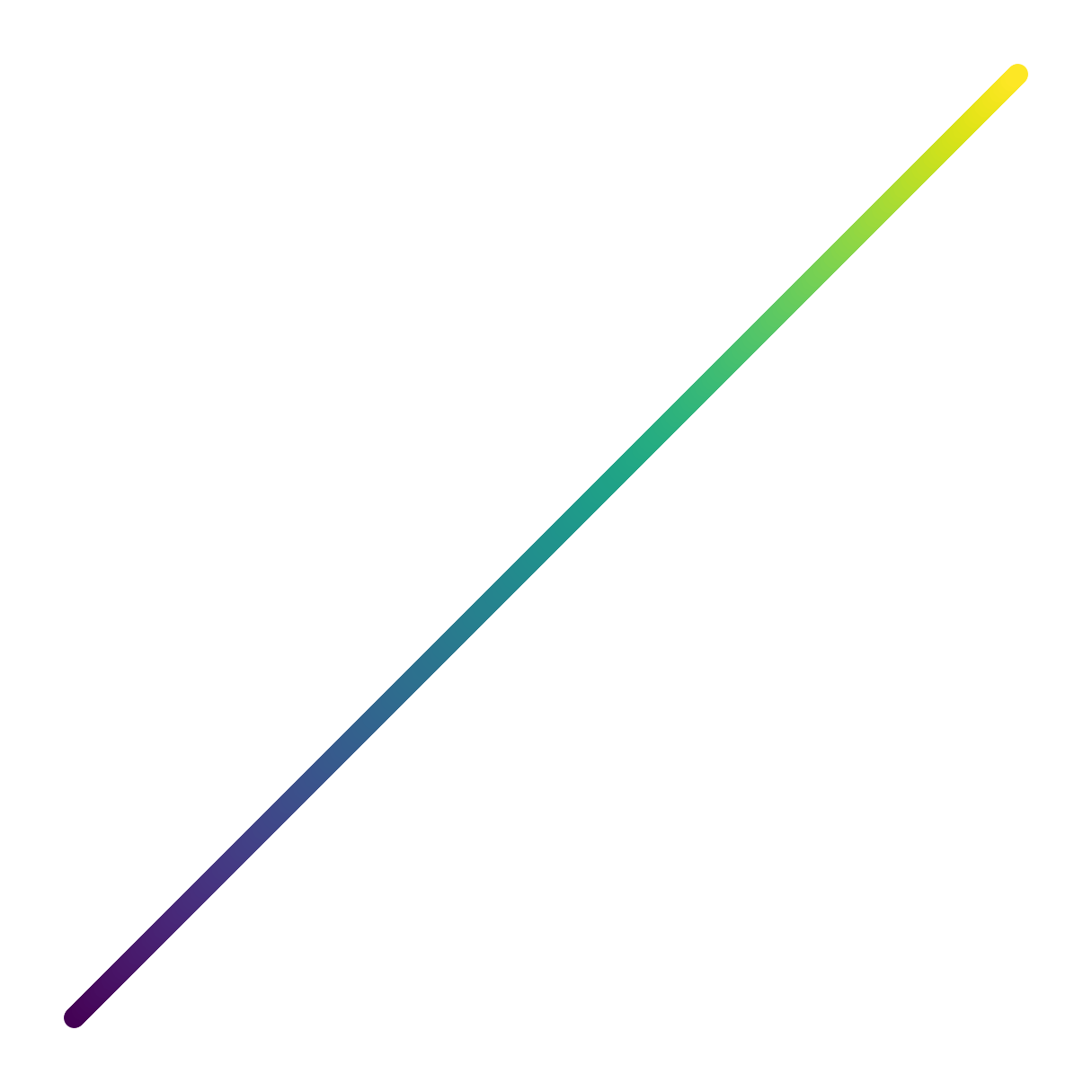}
        \caption{2D line}
    \end{subfigure}

    \caption{t-SNE plots for the 3-D Line Under Different Initializations}\label{fig:spaghetti}
\end{figure}


A few important items are of note: These plots vary in significant ways for different initializations: the points where the line segment twists or is broken are not consistent across the plots. The line, which from a topological perspective has only one component, has been broken into a large number of components. From a practitioner's perspective, this means that the groupings observed may not be associated with any real structure of the original data and may solely be artifacts of (local) minimization of the energy.  We also observe that the energy associated with the random initialization is slightly lower than that of the PCA initialization, although the PCA performs better than the random initialization in capturing the data’s topology.  The energy of the $2D$ line in $\mathbb{R}^{2}$, namely just by linearly projecting the original data, is also given for reference is much worse than the other runs. 
\end{example}

\begin{table}[h!]
\centering
\begin{tabular}{lrrr}
\toprule
 & Attraction Energy & Repulsion Energy & Total Energy \\
\midrule
Random     & 0.862646 & -6.344438 & -5.481791 \\
PCA        & 0.935205 & -6.410729 & -5.475525 \\
2-D line  & 0.000034 & -1.123033 & -1.122999 \\
\bottomrule
\end{tabular}
\caption{Energy comparisons across initializations}\label{tab:spaghetti}
\end{table}

\begin{example}[{``Grid Twisting''}]\label{ex:grid_twisting}
Further developing an example from \citep{wattenberg2016use}, we consider the setting where $X = Y = \R^2$, and our distribution in feature space is a square grid of equally spaced points. Here random initialization leads to qualitatively different outcomes: in this case two of the nine samples exhibit ``twisting'' phenomenon which break the grid into two well-separated pieces. All of the samples exhibit smaller-scale splitting in the grid.

\begin{figure}[H]
    \centering
    \begin{subfigure}{0.3\linewidth}
        \centering
        \includegraphics[width=\linewidth]{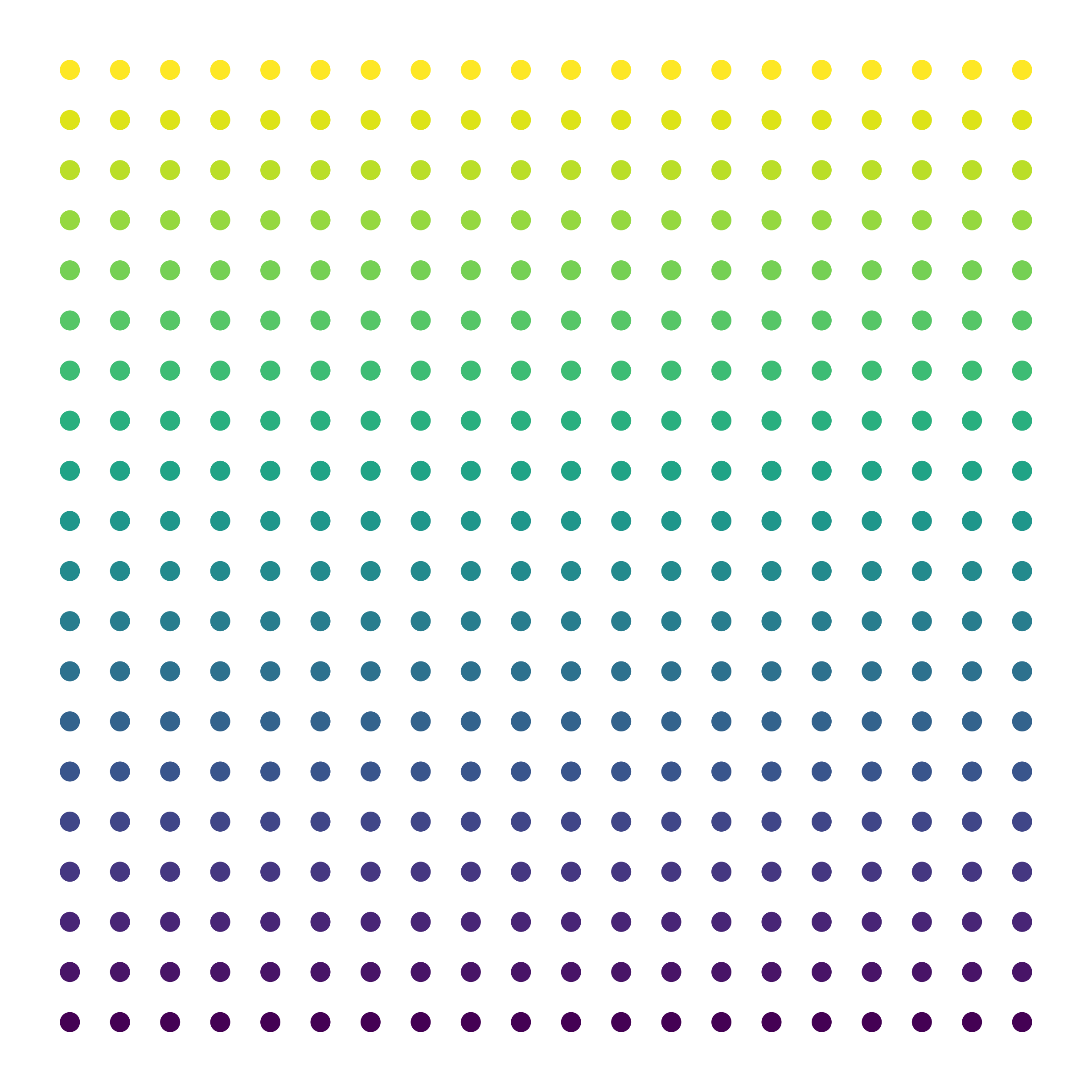}
        \caption{Original grid}
        \label{fig:grid_original}
    \end{subfigure}

    \vspace{0.5cm}

    \begin{subfigure}{0.9\linewidth}
        \centering
        \includegraphics[width=\linewidth]{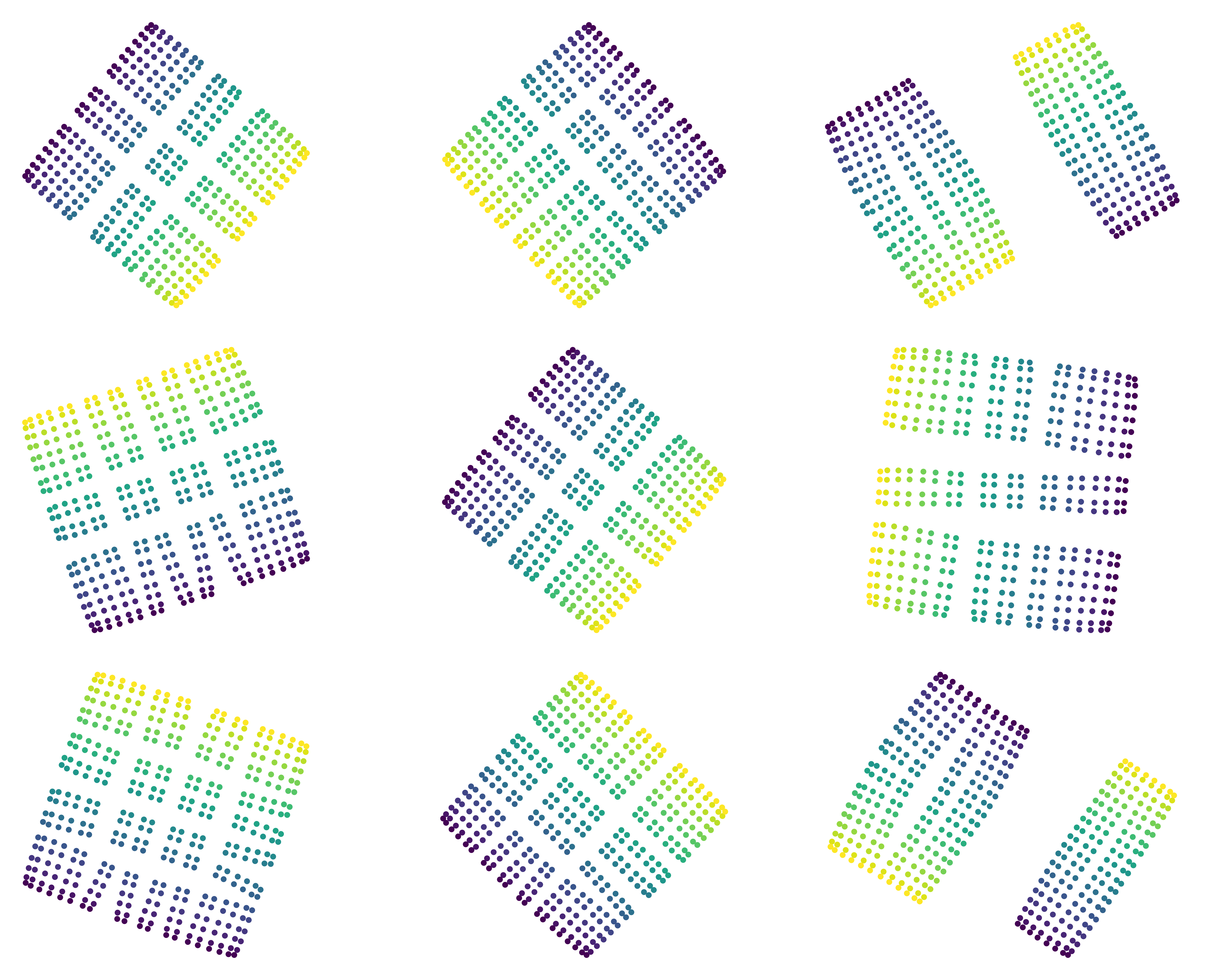}
        \caption{Random initialization with different seeds}
        \label{fig:grid_tsne}
    \end{subfigure}

    \caption{Outputs of t-SNE for a square grid of $N=400$ points with perplexity$=4$.}
    \label{fig:grid_side_by_side}
\end{figure}

\end{example}

\begin{example}[\emph{Swiss Roll}]
We note that \citep{linderman2017clusteringtsneprovably} also demonstrate analogous results on the classical Swiss Roll dataset. It is well known that t-SNE fails to unroll this manifold, recovering instead a set of disconnected clusters rather than the underlying two-dimensional sheet. In the various runs below, it is easy to see that apart from breaking clusters at various places, there is a form of ``twisting'' similar to Example \ref{ex:grid_twisting} which can be observed if you zoom in on the images.

\end{example}

\begin{figure}[h!]
    \centering

    \begin{subfigure}{0.35\linewidth}
        \centering
        \includegraphics[width=\linewidth]{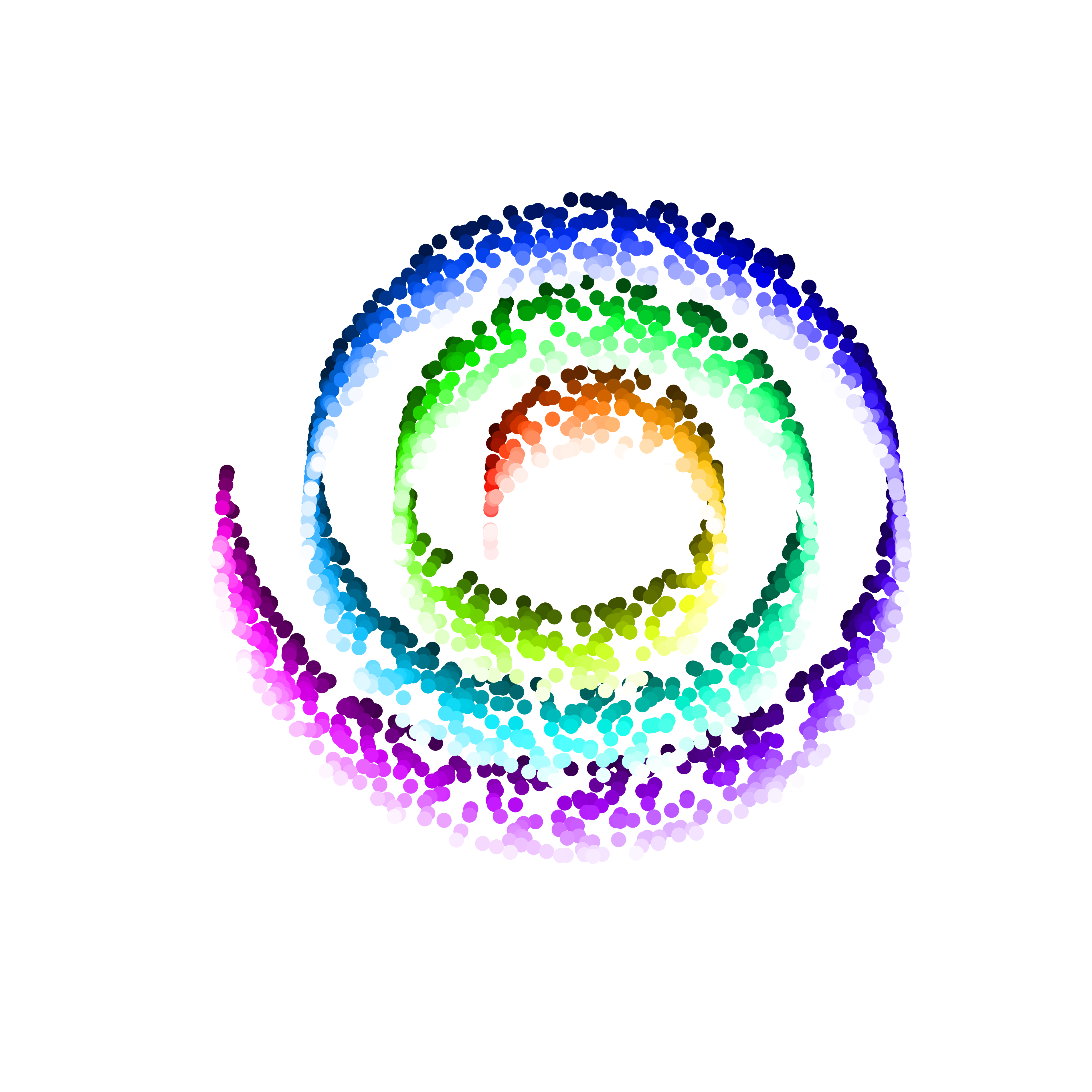}
        \caption{Original data for Swiss Roll in $\mathbb{R}^3$}
    \end{subfigure}
    
    \vspace{0.5cm}
    \begin{subfigure}{0.7 \linewidth}
        \centering
        \includegraphics[width=\linewidth]{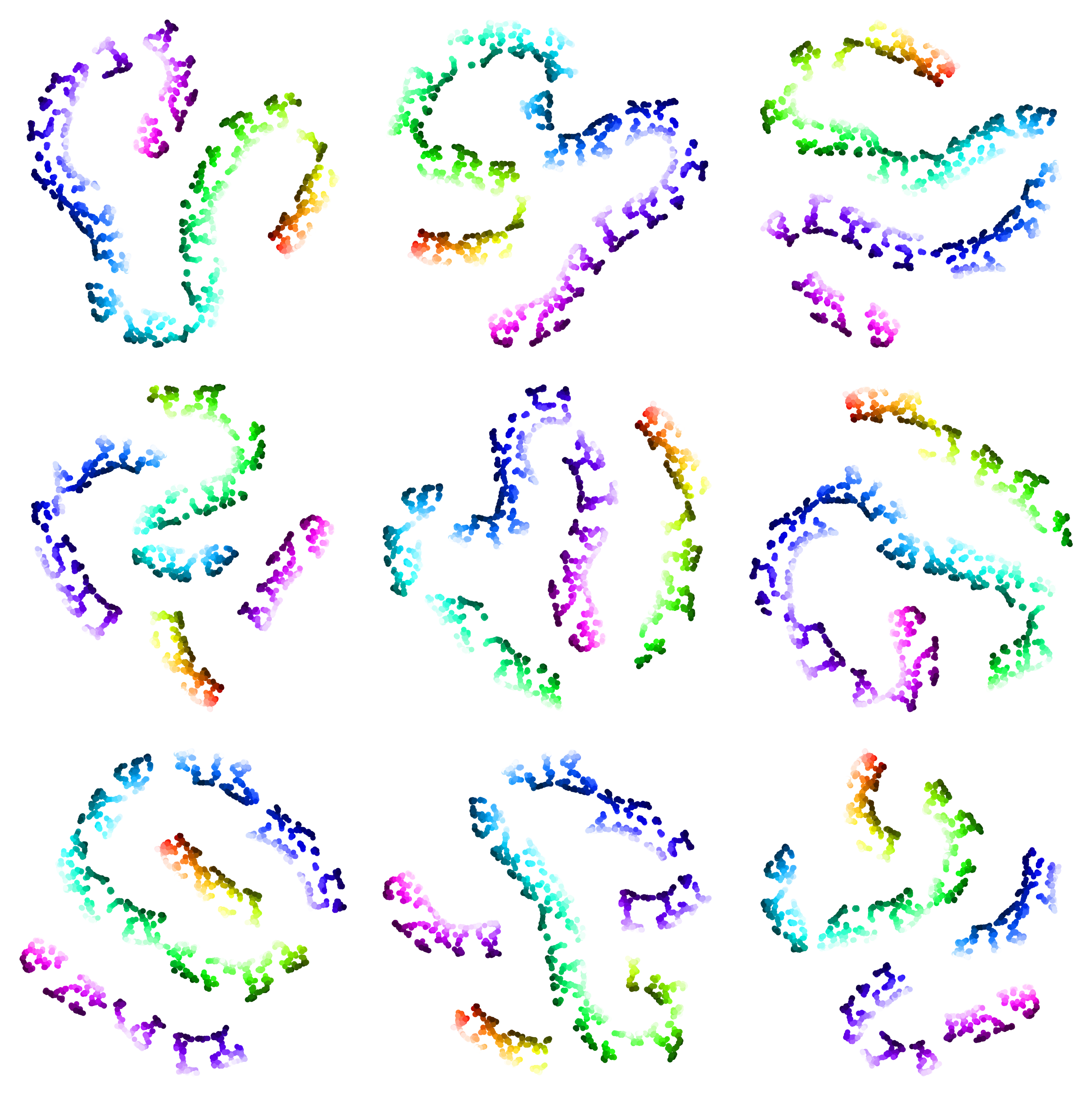}
        \caption{t-SNE runs with random initializations }
    \end{subfigure}

    \caption{Outputs of t-SNE for the Swiss Roll, with $N=2000$ and perplexity$=13$.}
\end{figure}

These examples reveal a significant complexity that is inherent in the energy landscape of t-SNE. Local minimizers exhibit a range of exotic behavior. The main goal of this paper is to provide some first steps explaining the energy landscapes which lead to these types of outcomes.

\section{Symmetry invariant embeddings}\label{sec:symmetries}

We now turn our attention to one fundamental property of the energy landscape of t-SNE: the fact that this energy respects symmetries of the underlying data. One of the main ideas of this paper is that this fact actually implies a high degree of complexity of the energy landscape, in the sense that it implies the existence of many distinct critical points. We mention that there is an extensive literature on the study of symmetry invariant energies in mathematical physics: for example the celebrated paper \citep{Palais1979}  uses many ideas similar to those in this work. One of the main goals of this paper is to extend those ideas meaningfully within the setting of dimension reduction algorithms.

As mentioned in Section \ref{sec:lit-review}, very few papers have used symmetries to study t-SNE, and those that do consider only scaling symmetries; this work considers a richer family. While our construction requires some symmetry assumptions on the input data in feature space, it still reveals a significant amount of information about energy landscapes, and begins to explain some of the examples in Section \ref{sec:Classic-Examples}. In this section we recall some classical definitions for symmetry groups, and develop appropriate definitions for their application to describing symmetry invariant embeddings for t-SNE. Throughout we also give illustrative examples for the reader's convenience.

\subsection{Finite isometry groups: definitions, examples and invariant mappings}
\textbf{Finite Isometry Groups:} To begin, an \emph{isometry} is a bijection between two metric spaces that preserves pairwise distances. In this work, we will always be considering isometries which map $X \to X$ or $Y \to Y$. Typical examples in Euclidean spaces include rotations, translations, reflections, and compositions thereof. A set of isometries $\Gc$ is called \emph{closed under compositions} if for any $g_1,g_2 \in \Gc$ we have that $g_1 \circ g_2$ is also in $\Gc$: we note that in our context, where the domain and range are the same, the composition is well-defined and is also an isometry. Such a finite family $\Gc$ is called a \emph{finite isometry group}; finiteness ensures $\Gc$ contains the identity and inverse of each of its elements,  with composition serving as the group operation.

We will assume that, associated with $X$ and $Y$, there are two finite groups of isometries $\Gc_X$ and $\Gc_Y$, and that the two groups are equivalent\footnote{This type of equivalence is generally called a \emph{group isomorphism}, but we don't emphasize that terminology since we are exclusively treating groups of isomorphisms in this work.}, in the sense that there is a bijection $\mathbb{G}: \Gc_X \to \Gc_Y$ so that for all $g,g' \in \Gc_X$ we have $\mathbb{G}[g] \circ \mathbb{G}[g'] = \mathbb{G}[g \circ g']$. The fact that $\Gc_X,\Gc_Y$ are equivalent implies that both groups have the same number of elements, which we denote by $k$. 

A convenient setting, which we utilize extensively, is the setting where the finite group is generated by a single element, called a generator; more explicitly a generator $g$ satisfies $\Gc = \{g^i, i=0 \dots k-1\}$, where $g^i$ denotes $i$ iterated compositions of $g$. 
We note that groups which admit a single generating element are called \emph{cyclic groups}.

To make things more concrete, we begin with several elementary examples.

\begin{example}\label{ex:torus-rotations}
    
    Let $X =Y =  \mathbb{S}^1$, and let $g_X$ be a shift by $2\pi/k$ radians. Assuming $k$ is prime, we can let $g_Y$ be a shift by $2\pi j /k$ radians for any choice of $j \in 1 \dots k-1$. We let $\mathbb{G}(g_X^i) = g_Y^i$ for all $i = 1 \dots k$. Here $\Gc_Y = \Gc_X$, and is given by the family of all shifts by an integer multiple of $2\pi/k$ radians.
\end{example}

The previous example naturally admits extension to Euclidean spaces.

\begin{example}\label{ex:R3-rotation}
    Let $X=\R^3$ and define
    \[
    g_X(x) = Ux,\qquad U = \begin{pmatrix} 1 & 0 & 0 \\ 0 & \cos(2\pi/k) & \sin(2\pi/k) \\ 0 & -\sin(2\pi/k) & \cos(2\pi/k) \end{pmatrix}.
    \]
    Here $\Gc_X$ has $k$ elements, corresponding to rotations in the $(x_2,x_3)$ directions, and is a subgroup of the classical ``special orthogonal'' group $SO(3)$, which can be described as the set of all orientation preserving rotations. We note that $SO(3)$ can also be represented in terms of all $3 \times 3$ orthogonal matrices with determinant $1$. 
\end{example}

While we gave the example in three dimensions, the same construction can be carried out in $\R^d$ with $d \geq 2$. The set of rotations is only a part of the family of isometries on Euclidean space, as evidenced by the next example.

\begin{example}\label{ex:reflection}
    Let $X = \R^2$, and define
    \[
    g_X(x) = U x,\qquad U = \begin{pmatrix} 1 & 0 \\ 0 & -1 \end{pmatrix}.
    \]
    Here $\Gc_X$ has two elements (namely $g_X$ and the identity), and corresponds to a reflection of the second coordinate of $x$. This is an element of the classical ``orthogonal group'' $O(2)$, which consists of all orthogonal matrices. 
\end{example}

\begin{example}\label{ex:dihedral}
    Not all symmetry groups are cyclic. A simple example is  the dihedral group (denoted by $D_k$), which encompasses both rotations and reflections in $\R^2$. Concretely, we can write elements of this group in matrix form via
\[
g(x) = Ux, \qquad U = \begin{pmatrix} \cos(2\pi j/k) & \sin(2\pi j/k) \\ -\sin(2\pi j/k) & \cos(2\pi j/k) \end{pmatrix} \begin{pmatrix} 1 & 0 \\ 0 & -1 \end{pmatrix}^i,
\]
where $j \in \{ 0 \dots k-1\}$ and $i \in \{0,1\}$. This group can  be generated by two of its elements, and is non-abelian for $k \geq 3$.
\end{example}


\textbf{Symmetry Invariant Couplings:} The central focus of this work is on the existence of \emph{symmetry solutions} to the critical point equations, for different choices of symmetry pairs. It turns out that the family of discrete isometries is consequently rich enough that one can then construct many (distinct) critical points associated with minimizing $\mathcal{C}$. More concretely, we say that a coupling $\pi$ is invariant under the tuple $(\Gc_X,\Gc_Y,\mathbb{G})$ if it satisfies
\begin{equation}\label{eqn:symm-inv-plans-def-gen}
    \pi(A\times B) = \pi(g(A) \times \mathbb{G}[g](B)), \forall g \in \Gc_X,
\end{equation}
for all Borel $A,B$.
In the case where both groups are cyclic with generators $g_X,g_Y$, we say (abusing notation slightly) that the coupling $\pi$ is invariant under the pair $(g_X,g_Y)$ if
\begin{equation}\label{eqn:symm-inv-plans-def}
    \pi(A\times B) = \pi(g_X(A)\times g_Y(B))
\end{equation}
for all Borel $A,B$: this corresponds to letting $\mathbb{G}[g_X] = g_Y$ and then extending via composition.

We note that setting $B = Y$ we immediately find that 
$
\mu_X(A) = \mu_X(g_X(A))$
is a necessary condition for any invariant $\pi$.
In the natural case when $\pi$ is induced by a map, namely when $\pi(dxdy) = \delta_{T(x)}(dy)\mu_X(dx)$, this symmetry relationship can be rewritten in terms of that map with the equation
\begin{equation}\label{eqn:symm-inv-maps-def-gen}
    T(g(x)) = \mathbb{G}[g](T(x))
\end{equation}
for all $g \in \Gc_X$, or in the cyclic case by writing
\begin{equation}\label{eqn:symm-inv-maps-def}
    T(g_X(x)) = g_Y(T(x)).
\end{equation}
The relationships \eqref{eqn:symm-inv-maps-def-gen}-\eqref{eqn:symm-inv-maps-def} only imply the invariance of the induced coupling under the additional requirement that $\mu_X\circ g_X = \mu_X$.

The definition that we give for symmetry invariant plans is analogous to the classical definition of an equivariant map, when in the cyclic case \( g_X = g_Y\). 
In that setting, the equation for maps can be interpreted as stating that \( T \) commutes with \( g_X \).

\textbf{Orbits and Fundamental Domains:} If a plan (or map) is symmetry invariant, it means that there are less degrees of freedom. In particular, by using language from abstract algebra, we can define a ``reduced domain'' on which we can define $\pi$ or $T$, and from which all other points can be inferred. Specifically, given a point $x \in X$, we can define the \emph{orbit} of $x$ by
\[
\mathcal{O}(x) = \{x' : x' = g(x), g \in \Gc_X\}.
\]
These orbits may contain at most $k$ elements. We call $E \subset X$ a \emph{fundamental domain}  if $\cup_{x \in E} \mathcal{O}(x) = X$ and $\mathcal{O}(x') \cap \mathcal{O}(x) = \emptyset$ for all $x \neq x' $ in $E$. One can view the fundamental domain as a particular choice of representatives of the quotient space $X / \mathcal{O}$, which is commonly denoted as $X / \Gc_X$. In many practical cases, we can select the  \emph{fundamental domain} to be connected. 
We can now give examples of these fundamental domains. 

\begin{example}
    Given the translation symmetry on the circle from Example \ref{ex:torus-rotations}, one choice of fundamental domain in $X$ is $[0,2\pi/k)$. Of course other choices are possible, such as $[2\pi j/k, 2\pi(j+1)/k)$.
\end{example}

\begin{example}
    Given the reflection symmetry in Example \ref{ex:reflection} one choice of fundamental domain would be $E = \{x_2 \leq 0\}$. We note that here the orbits along the $x_2 = 0$ axis are actually singletons, and make up the invariant set of $X$ under the group action.
\end{example}

We note that if we select a fundamental domain $E$ associated with $\Gc_X$, and if $T$ is invariant under the tuple $(\Gc_X,\Gc_Y,\mathbb{G})$, then $T$ is completely determined by its values on $E$. This is because, by the definition of the orbits, for any $x' \in X$ we can find a $g \in \Gc_X$ and a unique $x \in E$ so that $g(x) = x'$, and hence by the invariance $T(x') = \mathbb{G}[g](T(x))$.

The question of the existence of a symmetry invariant map is somewhat more delicate and depends upon both the algebraic properties of $\mathcal{G}_X,\mathcal{G}_Y$ in addition to the structure of $X,Y$, as evidenced by the following example.

\begin{example}
    Let $X = \R^2$, $Y = B( ( 0,-1),1/2) \cup B((0,1),1/2) \subset \R^2$, and let $\Gc_X = \Gc_Y$ be the group with two elements generated by the reflection about the $x$-axis. We note that in $X$ there are orbits of length $1$ (namely for all points on the $x$-axis), while in $Y$ there are only orbits of length $2$. This implies that there cannot be any symmetry invariant maps.
\end{example}



In the setting where $\Gc_X,\Gc_Y$ are generated by $g_X,g_Y$, the question of the existence on an invariant mapping $T$ has a simple resolution. In particular, given a fundamental domain $E_Y$ of $Y$, we can define
\begin{equation}
    E_{Y,s} := \{y \in E_Y : |\mathcal{O}(y)| = s\}.
\end{equation}
We notice that $E_{Y,1}$ is actually independent of the particular choice of fundamental domain $E_{Y}$, and consists of precisely points which are \emph{fixed points} under the group $\Gc$, that is $E_{Y,1} = \{y : g(y) = y  \: \forall g \in \Gc_Y\}$.

Now, if $k$ is prime and $E_X,E_Y$ are fundamental domains associated with $\Gc_X,\Gc_Y$, by using the orbit-stabilizer theorem and Lagrange's theorem, the only possible orbit lengths are $1$ or $k$, and thus we can decompose these fundamental domains into $E_{X,1} = \{x: |\mathcal{O}(x)| = 1\}, E_{X,k} = \{x : |\mathcal{O}(x)| = k\}$, such that $E_{X,1} \cup E_{X,k}=E_{X}$. It turns out that any mapping which maps $E_{X,1} \to E_{Y,1}$ and $E_{X,k} \to E_{Y,k}$ can be directly extended using the invariance relation \eqref{eqn:symm-inv-maps-def}. This means that a symmetry invariant map will exist as long as $\min(|E_{X,1}|,1) \leq |E_{Y,1}|$.

\textbf{Mutually Exclusive Embeddings:} We are interested in identifying a large family of \emph{distinct} critical points associated with different pairs of symmetries. In order to describe ``distinctness'' we say that two mappings $T_1,T_2$ are \emph{equivalent} if there exists isometries $\phi_X,\phi_Y$ so that $T_1(x) = \phi_Y(T_2(\phi_X(x)))$, and we say that two mappings are \emph{distinct} if they are not equivalent. 

Our construction of a multitude of critical points is based on the existence of many different symmetry pairs. It turns out that many symmetry pairs are essentially incompatible with each other, in the sense that they are always associated with different critical points. More precisely, we say that two tuples of symmetries $(\Gc_X,\Gc_Y,\mathbb{G}), (\tilde \Gc_X,\tilde \Gc_Y,\tilde{\mathbb{G}})$ are \emph{mutually exclusive} if any $T$ which is invariant under $(\Gc_X,\Gc_Y,\mathbb{G})$ is distinct from any $(\tilde \Gc_X,\tilde \Gc_Y,\tilde{\mathbb{G}})$ invariant $\tilde T$. In the cyclic case we use the same terminology but refer to pairs via their generators $(g_X,g_Y)$.

We start with an example which will be a building block for higher-dimensional settings.

\begin{example}\label{ex:torus-mutually-exclusive}
    Let $X = Y = \mathbb{T}$, $k \in \mathbb{N}$ and let $g_X=\tilde g_X$ be given by a shift by $2\pi/k$ radians. Let $0 < i, j < k$, with $i \neq j$, and choose $g_Y$ and $\tilde g_Y$ to respectively be shifts by $2\pi i/k$ and $2\pi j /k$ radians.

    Let us assume that a mapping $T$ is invariant under $(g_X,g_Y)$, meaning that $T(g_X(x)) = g_Y(T(x))$, or in other words that it satisfies \eqref{eqn:symm-inv-maps-def}, and that $\mu_X$ is uniform. We note that isometries on the $1-$torus are rather restrictive: namely they are only shifts or reflections. Let's suppose that $T$ is equivalent to some $(\tilde g_X,\tilde g_Y)$ invariant map $\tilde T$. Then we have
    \[
    \tilde g_Y \tilde T = \tilde T g_X = \phi_Y T \phi_X g_X = \phi_Y g_Y T g_X^{-1}\phi_X g_X = \phi_Y g_Y \phi_Y^{-1} \tilde T \phi_X^{-1} g_X^{-1} \phi_X g_X,
    \]
    where here we are dropping all the parentheses and ``products'' mean function composition.
    We can identify the torus with arithmetic modulo $2\pi$, and in that framework we can write $\phi_X(x) = a_X x + b_{X,\phi}$, $\phi_Y(y) = a_Y y + b_{Y,\phi}$ and $g_X(x) = x + 2\pi/k$, $g_Y(y) = y + 2\pi i /k$, $\tilde g_Y(y) = y + 2\pi j / k$, and with $a_Y, a_X = \pm 1$. After some algebra we have
    \[
    \tilde g_Y^{-1} \phi_Y g_Y \phi_Y^{-1}(y) = y + a_Y 2\pi i/k - 2\pi j/k.
    \]
    Similarly, we obtain
    \[
    \phi_X^{-1} g_X^{-1} \phi_X g_X(x) = x -a_X2\pi/k + 2\pi/k.
    \]
    There are a few cases to consider. If $a_X = 1$ then this gives
    \[
    \tilde T(x) = \tilde T(x) + a_Y2\pi i/k - 2\pi j/k.
    \]
    Unless $j = \pm i \, \emph{mod} \, k$ then this is impossible. Similarly, if $a_X = -1$ then we obtain
    \begin{align*}
    \tilde T(x) &= \tilde T(x + 2\cdot 2\pi/k) + a_Y2\pi i/k - 2\pi j/k \\
    &= \tilde T(g_X^2(x)) + a_Y2\pi i/k - 2\pi j/k  = \tilde g_Y^2(\tilde T(x)) + a_Y2\pi i/k - 2\pi j/k \\
    &= \tilde T(x) + a_Y2\pi i/k + 2\pi j/k.
    \end{align*}
    Again, unless $i = \pm j \, \emph{mod} \, k$ then this is impossible. 

    This then implies that if $0<i<j<k/2$ the two symmetry pairs are mutually exclusive.
\end{example}

It is useful to consider an example of a simple situation where symmetry pairs are not mutually exclusive.

\begin{example}
    Suppose that $\Gc_Y$ and $\tilde \Gc_Y$ share the same set of points with orbit length 1, namely $E_{Y,1} = \tilde E_{Y,1}$. One example when $Y = \R^3$ could be any finite rotation group which leaves the $x_1 $axis unchanged, as in Example \ref{ex:R3-rotation}. For any choice of symmetry groups $\Gc_X,\tilde \Gc_X$ then any map from $X \to E_{Y,1}$ will be invariant under both symmetry pairs, meaning the symmetries are not mutually exclusive.
\end{example}

As evidenced by the previous example, certain symmetries, especially rotation symmetries, have non-trivial sets of fixed points, which directly rules out the possibility of mutual exclusivity. However, as we'll see in Section \ref{sec:distinctness}, energetic considerations can sometimes exclude these trivial maps. We'll focus our attention, for the purposes of definition, on symmetries where $E_{Y,1}$ is at most a singleton. Given two symmetry pairs $(g_X,g_Y)$, $(\tilde g_X,\tilde g_Y)$ such that the $E_{Y,1}$ associated with both $g_Y,\tilde g_Y$ are the same singleton, we say that the pair of symmetries are mutually exclusive up to constants if any non-constant $T$ which is $(g_X,g_Y)$ invariant is distinct from any $(\tilde g_X,\tilde g_Y)$ invariant $\tilde T$.

\begin{example}\label{ex:mutually-exclusive-R2}
    Consider $X = Y=\R^2$, $k \in \mathbb{N}$, and let $g_X = \tilde g_X$ be rotations by $2\pi/k$ radians. Building upon Example \ref{ex:torus-mutually-exclusive}, let $g_Y$ be a rotation by $2\pi i /k$ radians and $\tilde g_Y$ be a rotation by $2\pi j /k$, with $0 < i<j<k/2$. Here again we consider symmetry invariant mappings which satisfy \eqref{eqn:symm-inv-maps-def}, and in order to connect this with \eqref{eqn:symm-inv-plans-def} one would need to assume that $\mu_X$ is invariant under $g_X$: a sufficient condition for any choice of $k$ would be that $\mu_X$ is only dependent upon $|x|$..

    We note that in two dimensions, the family of symmetries consists of rotations and reflections (which we already considered in Example \ref{ex:torus-mutually-exclusive}), as well as translations. As we have selected $g_X,g_Y,\tilde g_Y$ based upon rotations, choosing $\phi_X,\phi_Y$ based upon translations will violate the invariance condition as long as $T$ is non-constant. On the other hand, for $\phi_X,\phi_Y$ based upon rotations and translations, all of the same computations apply, with the exception that the point $0$ is invariant under both $g_Y$ and $\tilde g_Y$. Hence there is one mapping, namely $T_0(x) = 0$, which is invariant under both symmetry pairs, but no other $(\tilde g_X,\tilde g_Y)$ invariant $\tilde T$ will be equivalent to any $(g_X,g_Y)$ invariant $T$. Hence $(g_X,g_Y)$ and $(\tilde g_X,\tilde g_Y)$ are mutually exclusive up to constants.
\end{example}

The extension of Example \ref{ex:mutually-exclusive-R2} to $Y = \R^m, m > 2$ requires the construction of a cyclic group on $Y$ for which $E_{Y,1}$ is a singleton (so that functions which map into that set are constants). We give an example of such a group action, which then allows the construction of infinitely many groups which are mutually exclusive up to constants in arbitrary dimension.

\begin{example}\label{ex:rotoreflections}
 Consider the family of isometries generated by a rotation of $2\pi/k$ in the first two coordinates, and a simultaneous reflection of all the remaining coordinates. In matrix form, this is given by
 \[
 g_Y(y) = Uy, \qquad U =  \begin{pmatrix} \cos(2\pi/k) & \sin(2\pi/k) & 0 \\ -\sin(2\pi/k) & \cos(2\pi/k) & 0 \\ 0 & 0 & -I_{m-2}\end{pmatrix},
 \]
 where the $0$'s are vectors of appropriate size, and $I_{d}$ is the $d \times d$ identity matrix. 

 If $k$ is even, this group is cyclic of order $k$, and if $k$ is odd then this group is cyclic of order $2k$. In any case, the only point which is invariant under this symmetry is the origin. These groups are sometimes called \emph{rotoreflections}. 
\end{example}

A similar approach also applies to the setting where $Y = \mathbb{S}^{m-1}$. We also note here that while we have been able to characterize certain classes of symmetries which are mutually exclusive under maps $T,\tilde T$, the family of ``trivial'' correspondences $\pi$ which are invariant under multiple symmetries is much larger. This is because the way we proved that symmetry pairs were mutually exclusive was based upon arguing that, in the torus case, $T(x) \neq T(x) + c$, whereas in the setting with correspondences multi-valuedness is permissible. The following example illustrates this obstruction on the torus.

\begin{example}\label{ex:correspondence-symmetries}
    Let $X = Y=\mathbb{T}$. Let $\pi$ be a uniform measure on $\mathbb{T} \times \mathbb{T}$. Then $\pi$ is invariant under any pair of isometries $(g_X,g_Y)$.
\end{example}

We finally remark that in our main results we will also require a type of invariance of $\sigma$, namely that 
\begin{equation}\tag{IV}\label{assump:sigma_invariance}
 \sigma\circ g_X = \sigma   
\end{equation} This is natural in the context of the implicit choices of $\sigma$ utilized in the original algorithm under the assumption that $\mu_X$ is also $g_X$ invariant. It would also hold in cases where $X$ is compact and $\sigma \equiv C$.

\subsection{Gradient descent for symmetry invariant embeddings}

The main idea of this work lies in the observation that gradient descent does not break symmetry invariance. 
If the symmetry is to be respected at all time t, following \eqref{eqn:symm-inv-plans-def} we need to show, for all $g \in \mathcal{G}_X$, that
\[
\pi_t(g(A) \times \mathbb{G}[g](B))=\pi_t(A \times B).
\]
For ease of notation, we write $G(A,B) := g(A) \times \mathbb{G}[g](B)$. Assuming that the invariance holds for our initial data, and using the fact that $\pi_t = \pi_0 \circ \Phi_t^{-1}$, we can rewrite the left and right hand side of the invariance condition as follows:
\begin{align*}
    \pi_t(G(A,B))&=\pi_0(\Phi_t^{-1} \circ G(A,B)), \\ 
    \pi_t(A,B)&=\pi_0(\Phi_{t}^{-1}(A,B))=\pi_0(G \circ \Phi_t^{-1}(A,B)).
\end{align*}
Hence the invariance condition is satisfied if $G \circ \Phi_t^{-1}  = \Phi_t^{-1} \circ G$, which, after composing with $\Phi$ on the right and left, is equivalent to $G \circ \Phi_t=\Phi_t \circ G$. This gives a condition that we can directly verify. Before doing so, we recall a standard result from Riemannian geometry.



\begin{lemma}\label{lem:isometry-lemma} Let $Y$ be a smooth, connected and complete Riemannian manifold, with metric $d_Y$, and let $g$ be an isometry on $Y$. Define $\eta_Y(y,y')=\xi(d_Y(y,y')^2)$.
   Then,\[
\partial_{1}\eta_Y(g(y),g(y'))=Dg(y)\partial_1\eta_Y(y,y').
\]
\end{lemma}
\begin{proof}
Let $z:[0,1] \to Y$ be a constant speed, length-minimizing geodesic from $y$ to $y'$. One directly computes
\[
\partial_{1}d_Y^2(y,y')=-2z'(0).
\]
Since $g$ is an isometry, we know that one length-minimizing, unit speed path between $g(y)$ and $g(y')$ is given by $g(z(t))$, which we denote by $\tilde{z}(t)$. We then compute via the chain rule
\[
\partial_1\eta_Y(g(y),g(y')) = -2\xi'(d_Y^2(g(y),g(y')))\tilde z'(0) = -Dg(z(0)) 2\xi'(d_Y^2(g(y),g(y')))z'(0) = Dg (y)\partial_1\eta_Y(y,y').
\]
\end{proof}

We now prove one of our main results, regarding the fact that gradient descent preserves symmetry invariance.

\begin{theorem}\label{thm:sym-invar-preserved}
Let $X$ and $Y$ be smooth manifolds and let $\pi_0 \in \Pi(\mu_X)$. Let $\Gc_X,\Gc_Y$ be finite groups of isometries on $X$ and $Y$ respectively, each having $k$ elements. 
Suppose that $\pi_t$ is a gradient descent curve associated with the t-SNE energy, meaning that it satisfies \eqref{eqn:GD-def} for $t \in [0,t_1)$. Assume that $\sigma(g(x)) = \sigma(x)$ for all $g \in \mathcal{G}_X$. Suppose that  $\pi_0$ is invariant under the tuple $(\mathcal{G}_X,\mathcal{G}_Y,\mathbb{G})$, meaning that for all $g \in \mathcal{G}_X$
\[
    \pi_0(g(A) \times \mathbb{G}[g](B))=\pi_0(A \times B).
\]
Finally, assume that \ref{assump-symmetric_eta},\ref{assump-eta_lipshitz} and \ref{assump-Xi} hold and that $I^{-1}$ is locally bounded in $[0,t_1)$. Then we have that for any $t \in [0,t_1)$ that $\pi_t$ is also invariant under $(\mathcal{G}_X,\mathcal{G}_Y,\mathbb{G})$. 


\end{theorem}

We remark that the assumptions on $\Xi$ and $\eta$ are relatively mild: bounded $\Xi$ is related to non-degenerating $\sigma$, which is generally true for finite data, whereas the assumptions on $\eta$ hold for functions based upon squared distances or for any compact manifold. Finally, we justify the bound on $I^{-1}$ in Lemma \ref{lem:I-bound} for solutions to the gradient descent equation.
\vspace{0.5 em}
\begin{proof}
Here for ease of notation, we write $g = g_X, \mathbb{G}[g] = g_Y$, as in the cyclic case, but if we replace $g_X$ with arbitrary $g \in \mathcal{G}_X$ and $g_Y$ with $\mathbb{G}[g]$ all of the formulas remain the same. The main idea of the proof is to estimate  the commutator
\[
W_t=G \circ \Phi_t-\Phi_t \circ G.
\]
In particular, if $W_t \equiv0$ $\forall t$, independent of $g$, then we have symmetry invariance. For convenience, we write $W_t(x,y)_x$ to denote the ``$x$'' coordinates of $W$ and $W_t(x,y)_y$ to denote the $y$ coordinates.

Clearly, $W_0(x,y) = 0$ as $\Phi_0 = Id$. Furthermore, for all $t>0$ we have $(W_t)(x,y)_x\equiv 0$ as $\Phi_t(x,y)_x = x$. Thus we only need to estimate $(W_t)(x,y)_{y}=g_Y(\phi_t(x,y))-\phi_t(g_X(x),g_Y(y))$. Thus we take a time derivative
\begin{equation}\label{eq:time-deriv-commutator}
\frac{d}{dt} W_t(x,y)_y=Dg_Y(\phi_t(x,y)) \Ec[\pi_t](x,\phi_t(x,y))-\Ec[\pi_t][g_X(x),\phi_t(g_X(x),g_Y(y)].
\end{equation}
We recall that we can decompose $\mathcal{E}[\pi_t] = \mathcal{K}_{\mathcal{A}}[\pi_t] + \mathcal{K}_{\mathcal{R}}[\pi_t]$. Here to ease in readability, we will primarily focus on the necessary transformations for the attraction term $\mathcal{K}_{\mathcal{A}}$, as the arguments are the same for the repulsion. Using our definition of $\mathcal{K}_{\mathcal{A}}$, along with the push-forward definition of $\pi_t$ and Lemma \ref{lem:isometry-lemma}, we may write


\begin{align*}
&Dg_Y(\phi_t(x,y))\mathcal{K}_\mathcal{A}[\pi_t](x,\phi_t(x,y))\\
&=\int_{X \times Y} \Xi(x,x')\frac{Dg_Y(\phi_t(x,y))\partial_{1}\eta_Y(\phi_t(x,y),y')}{(1+\eta_Y(\phi_t(x,y),y')) } \pi_t(dx'dy')  \\
&=\int_{X \times Y}  \Xi(x,x')\frac{\partial_1\eta_Y(g_Y(\phi_t(x,y)),g_Y(\phi_t(x',y')))}{(1+\eta_Y(g_Y(\phi_t(x,y)),g_Y(\phi_t(x',y'))) } \pi_0(dx'dy').
\end{align*}

Again, using the pushforward definition of $\pi_t$, along with the symmetry invariance of $\pi_0$ and $\Xi$ (which is inherited from $\sigma$), we may write
\begin{align*}
&\mathcal{K}_{\mathcal{A}}[\pi_t](g_X(x),\phi_t(g_X(x),g_Y(y))\\
&=\int_{X \times Y} 
\Xi(g_X(x),x')\frac{\partial_1\eta_Y(\phi_t\circ G(x,y),\phi_t(x',y'))}{(1+\eta_Y(\phi_t \circ G(x,y),\phi_t(x',y')) } \pi_0(dx'dy')  \\
&=\int_{X \times Y} 
\Xi(g_X(x),g_X(x'))\frac{\partial_1\eta_Y(\phi_t\circ G(x,y),\phi_t\circ G(x',y'))}{(1+\eta_Y(\phi_t \circ G(x,y),\phi_t \circ G(x',y')) } \pi_0(dx'dy')  \\
&=\int_{X \times Y} 
\Xi(x,x')\frac{\partial_1\eta_Y(\phi_t\circ G(x,y),\phi_t\circ G(x',y'))}{(1+\eta_Y(\phi_t \circ G(x,y),\phi_t \circ G(x',y')) } \pi_0(dx'dy').
\end{align*}

We note that the expressions in the integrals now only differ in that one is evaluated at $\phi \circ G$ while the other is written in terms of $G \circ \phi$, which facilitates estimation in terms of $W$ itself. We also note that the exact same transformations apply for the repulsion energy, with the notable difference that the repulsion term is multiplied by $I^{-1}$.


Now, by the assumption \ref{assump-eta_lipshitz}, we have that $\frac{\partial_1\eta}{1 + \eta}$ and $\frac{\partial_1\eta}{(1 + \eta)^2}$ are both Lipschitz in both of their arguments. Furthermore, by assumption we have that $I(\pi_t)^{-1}$ is locally bounded in $t$. Finally, we also have that $\Xi$ is bounded by \ref{assump-Xi}. In turn, for any interval $[0,T]$, there exists an $M$ such that
\[
\left|\frac{d}{dt} W_t(x,y)_y\right| \leq M\|W_t\|_\infty.
\]
We note that here we use $|\cdot |$ to denote lengths in $Tan_Y$. Using Gr\"onwall, we then obtain that $W_t \equiv 0 $ for all $t$, concluding the proof.
\end{proof}

The previous result means that for any symmetry invariant initial data, if we follow gradient descent then the embedding remains symmetry invariant for all times. This, along with the discussion of mutually exclusive embeddings in the previous subsection, gives a natural avenue for constructing a large family of distinct critical points of the energy. We first explore this phenomenon numerically, drawing links with previous examples, and then offer rigorous proofs under appropriate assumptions in Sections \ref{sec:descent-paths} and \ref{sec:distinctness}.

\subsection{Numerical Illustrations}\label{sec:num-examples}
In order to illustrate the utility of these symmetry-based methods in constructing different critical points, we provide various numerical illustrations. In all of the examples in this section we consider invariant maps, and not general correspondences. Our first two examples focus on symmetry pairs on the plane which are mutually exclusive up to constants.

To quantify the extent of symmetry in the t-SNE plot, we measure how far the
embedding is from being symmetry invariant. Each dataset is built as a group orbit, so
$g_X$ permutes the points and the permutation $\sigma$ with
$x_{\sigma(i)}=g_X(x_i)$ is known. Reading the embedding as the map
$T:x_i\mapsto y_i$, we compare $T(g_X(x))$ with $g_Y(T(x))$ point by point and
take the mean,
\[
\mathcal{SE}(T)\;=\;\frac{1}{\operatorname{diam}(Y)\,n}\sum_{i=1}^{n}
| T(g_X(x_i)) - g_Y(T(x_i)) |
\;=\;\frac{1}{\operatorname{diam}(Y)\,n}\sum_{i=1}^{n}
| y_{\sigma(i)} - g_Y\,y_i| 
\]
where $\operatorname{diam}(Y)=\max_{y_i,y_j}|y_i-y_j|$ and $g_Y$ is the element
paired with $g_X$. We normalize  by $\operatorname{diam}(Y)$, and it vanishes exactly when $T$ is symmetry invariant.

\begin{example}\label{example:S1_cricle_rotation_2pi/5}
Here $X=\Sbb^{1}$ and $Y=\mathbb{R}^2$. The group $\Gc_X$ is generated by rotation
by $\frac{2\pi}{5}$ and $\Gc_Y$ by rotation by $\frac{6\pi}{5}$. The fundamental
domain is $[0,\frac{2\pi}{5})$. We use the function $T(x,y)=(x^2,y)$, as an initialization on that fundamental domain, and extend using the symmetry to a function on all of $\Sbb^{1}$

We run two optimisers from this same initialisation. The first is vanilla
gradient descent, in which the exact gradient is used at every step. Since the
invariance is preserved at every iteration, the embedding remains symmetric ($\mathcal{SE}=3.088\times10^{-15}$); see
Figure~\ref{fig:s1_vanilla}.

The second is openTSNE, which introduces several approximations beyond the
vanilla objective: the affinity matrix is sparsified to the
$\lfloor 3\times\text{perplexity}\rfloor$ nearest neighbors, the repulsive
forces are approximated by a Barnes--Hut quadtree with parameter $\theta\ge 0$
($\theta=0$ recovering the exact gradient), and the optimizer applies momentum
and per-coordinate adaptive gains. None of these operations is symmetry invariant, and
the resulting defects are larger by many orders of magnitude
(Figure~\ref{fig:s1_sweep} and Table~\ref{tab:s1-combined}), across every
combination of $\theta$ and early exaggeration we tried. 

 We do note that openTSNE in this context attains \emph{lower} total energy than vanilla gradient descent. In many cases the outcomes maintain a high degree of visual symmetry, and still a relatively low value of $\mathcal{SE}$, and understanding whether the symmetry solutions are stable critical points or saddle points is an interesting question that we do not resolve here.
The code can be found on the authors' GitHub.
\footnote{\url{https://github.com/nak-har/abundance_of_critical_points_tsne}}
\end{example}

\begin{figure}[h!]
    \centering
    \begin{subfigure}{0.23\linewidth}
        \centering
        \includegraphics[width=\linewidth]{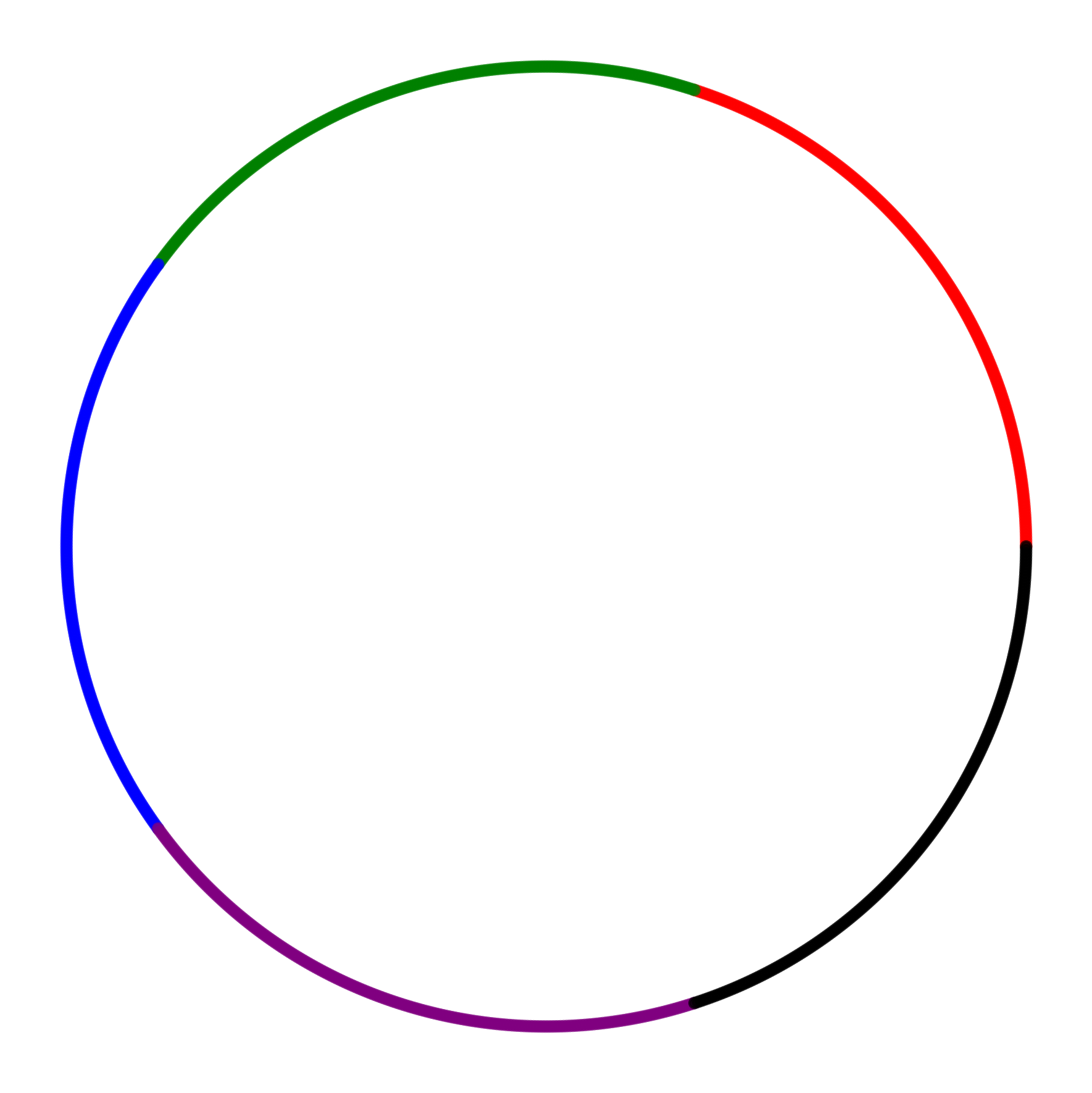}
        \caption{Points on the circle}
    \end{subfigure}
    \hfill
    \begin{subfigure}{0.23\linewidth}
        \centering
        \includegraphics[width=\linewidth]{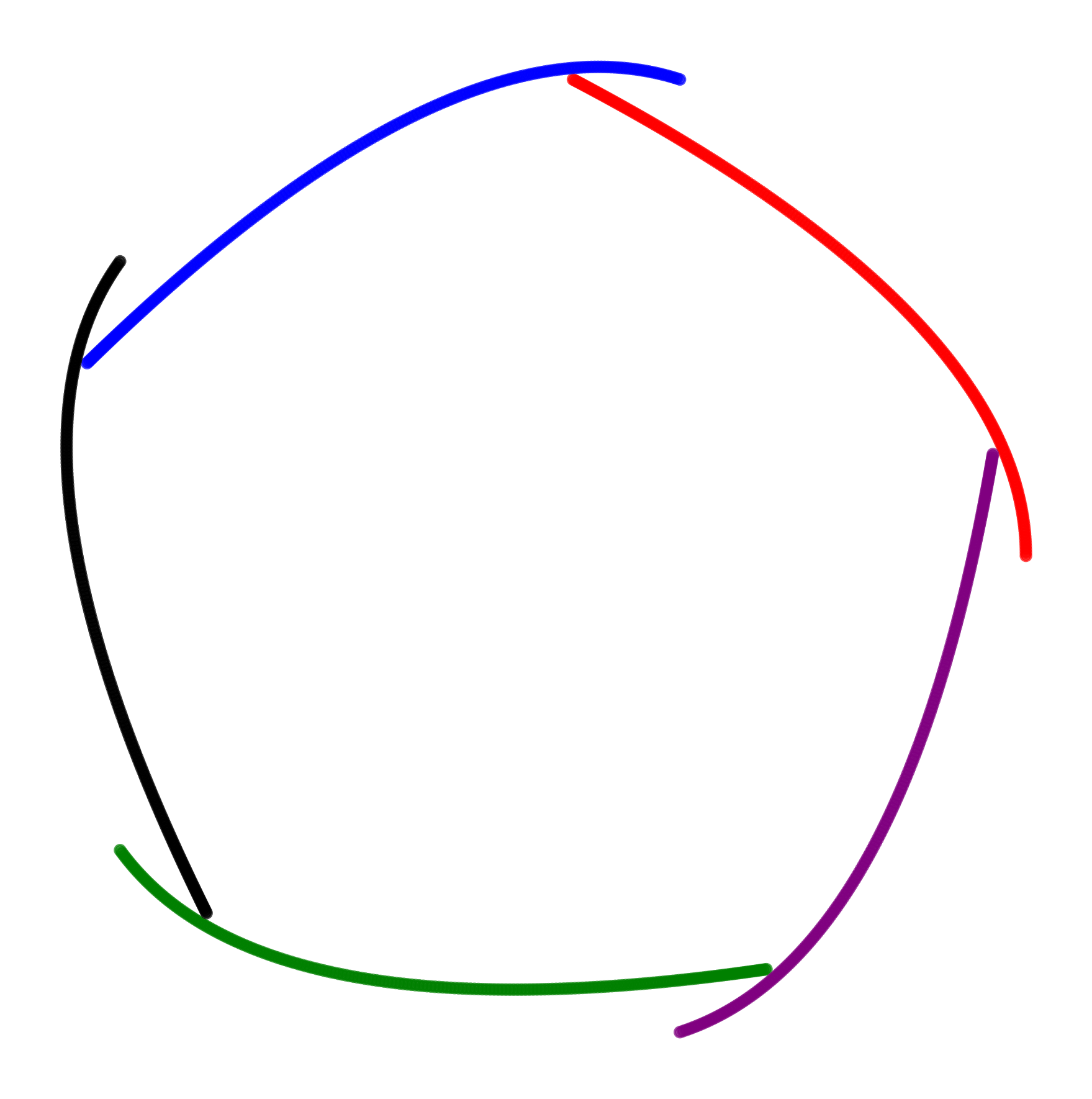}
        \caption{Initialisation}
    \end{subfigure}
    \hfill
    \begin{subfigure}{0.23\linewidth}
        \centering
        \includegraphics[width=\linewidth]{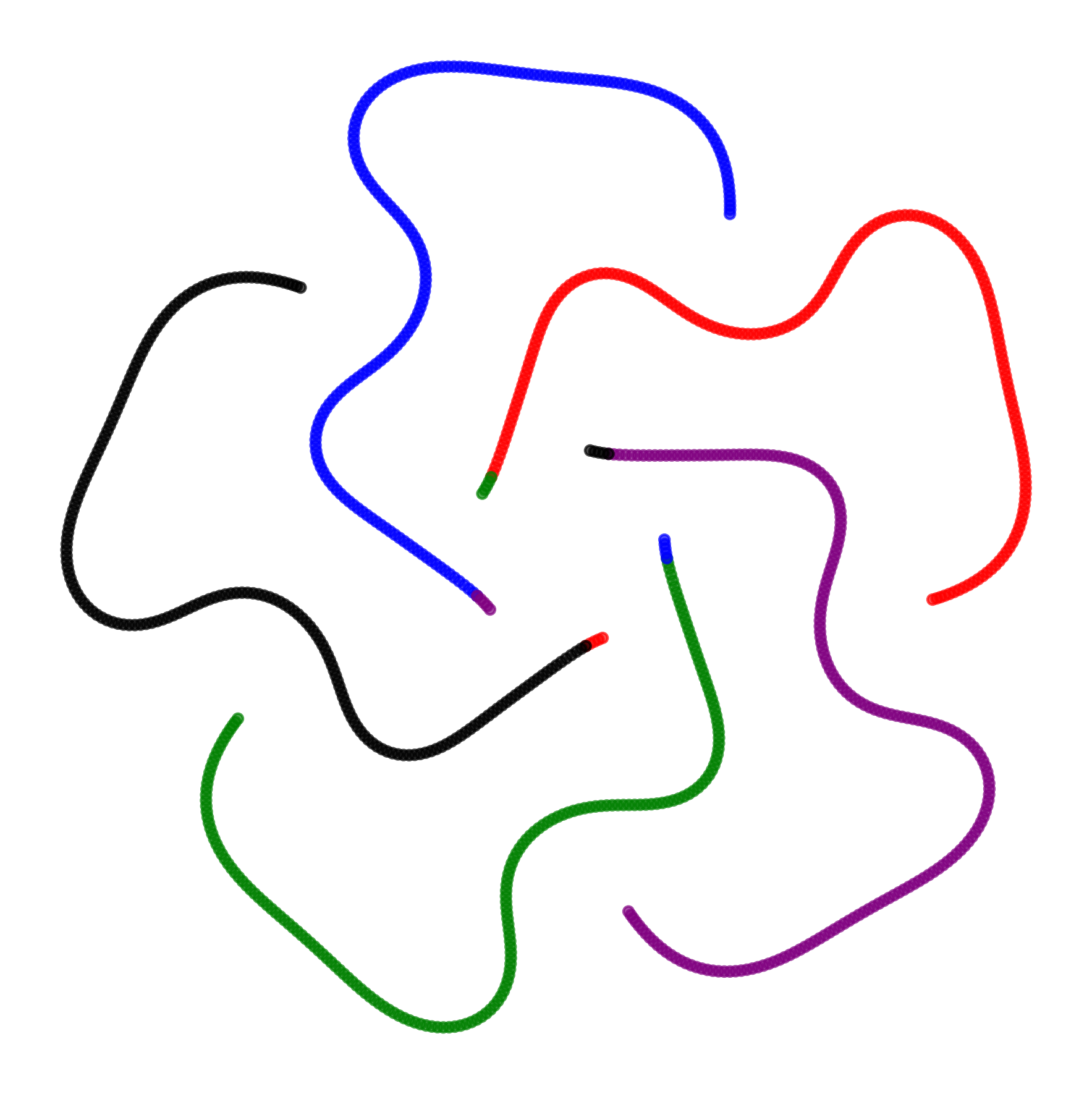}
        \caption{Vanilla gradient descent}
    \end{subfigure}
    \caption{Input data, symmetric initialisation, and the embedding produced by
    vanilla gradient descent for $\mathbb{S}^1$ with $T(x,y)=(x^2,y)$, $N=1000$, perplexity $10$. The symmetry of the initialization is preserved
    exactly ($\mathcal{SE}=3.088\times10^{-15}$).}
    \label{fig:s1_vanilla}
\end{figure}

\begin{figure}[H]
    \centering
    \begin{subfigure}{0.23\linewidth}
        \centering
        \includegraphics[width=\linewidth]{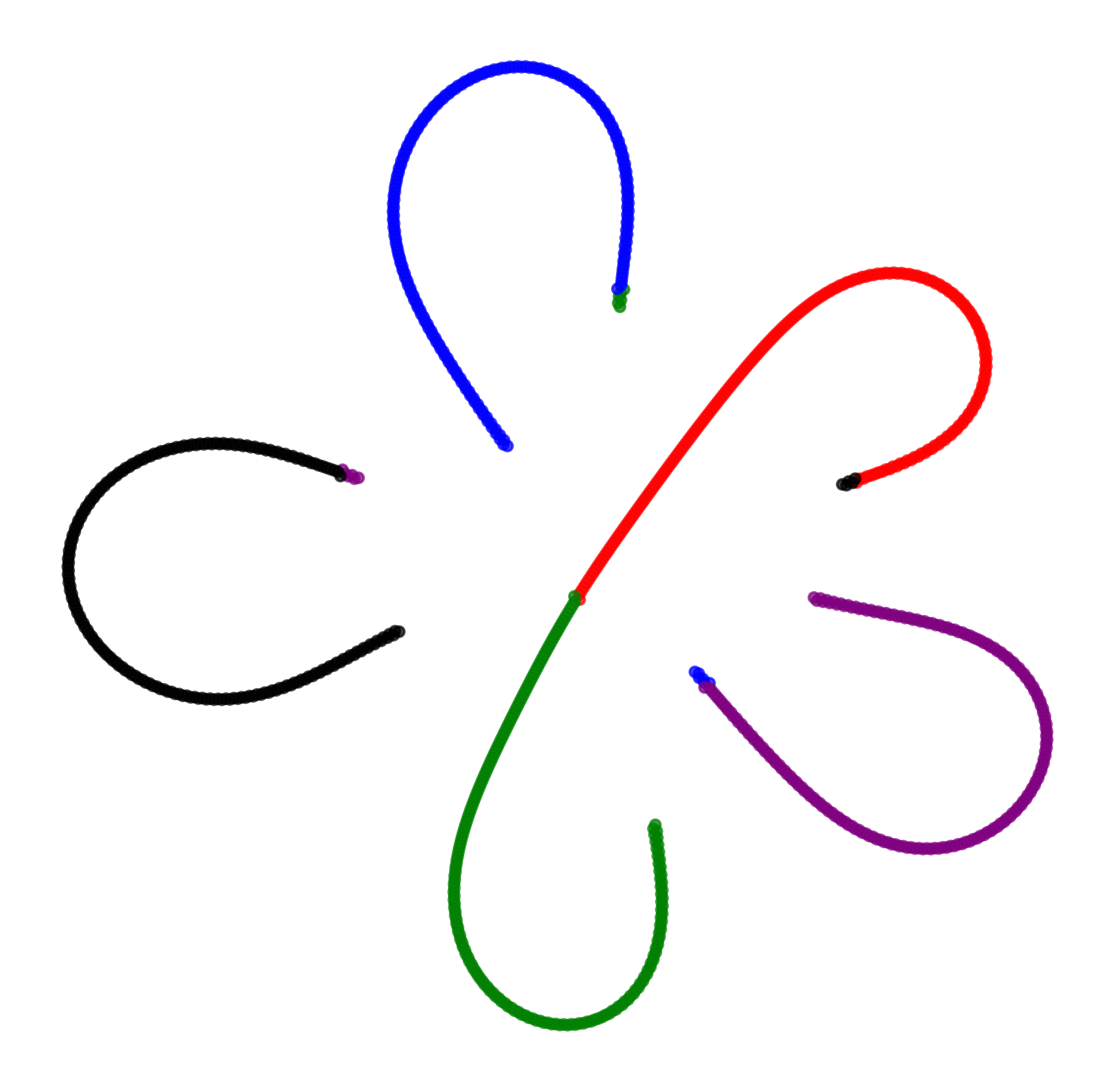}
        \caption{exag ON, $\theta=0$}
    \end{subfigure}
    \hfill
    \begin{subfigure}{0.23\linewidth}
        \centering
        \includegraphics[width=\linewidth]{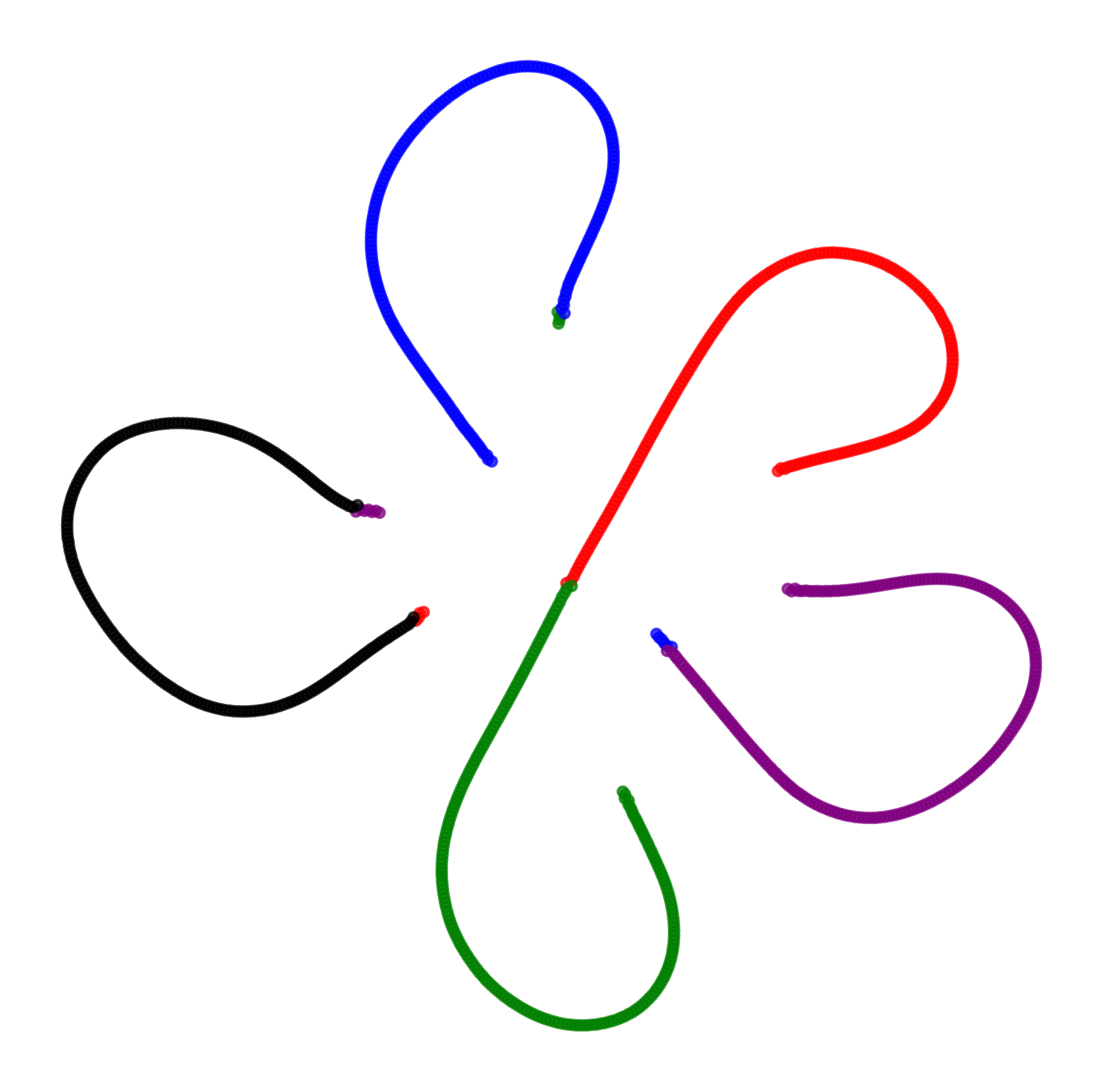}
        \caption{exag ON, $\theta=0.8$}
    \end{subfigure}
    \hfill
    \begin{subfigure}{0.23\linewidth}
        \centering
        \includegraphics[width=\linewidth]{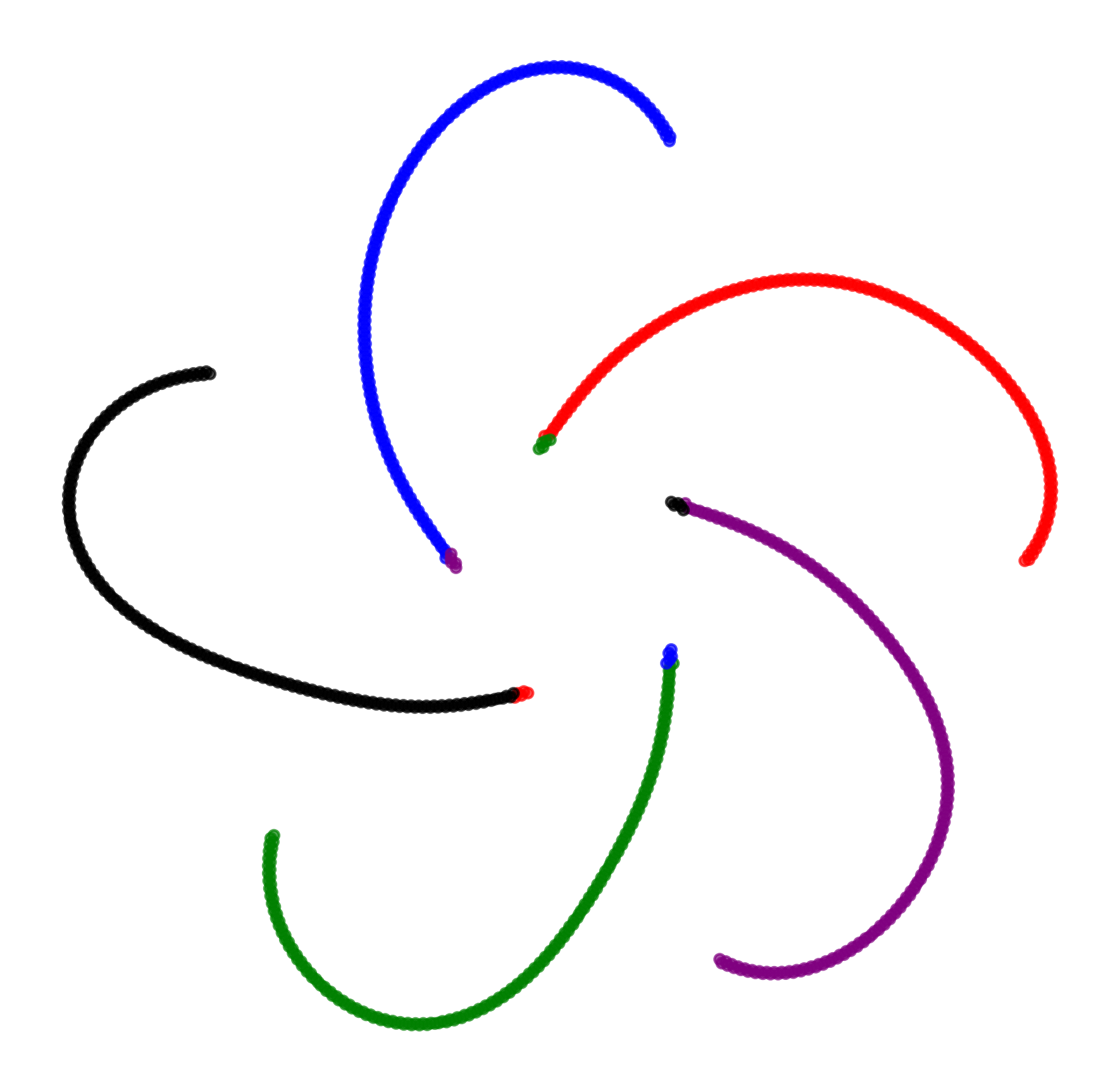}
        \caption{exag OFF, $\theta=0$}
    \end{subfigure}
    \hfill
    \begin{subfigure}{0.23\linewidth}
        \centering
        \includegraphics[width=\linewidth]{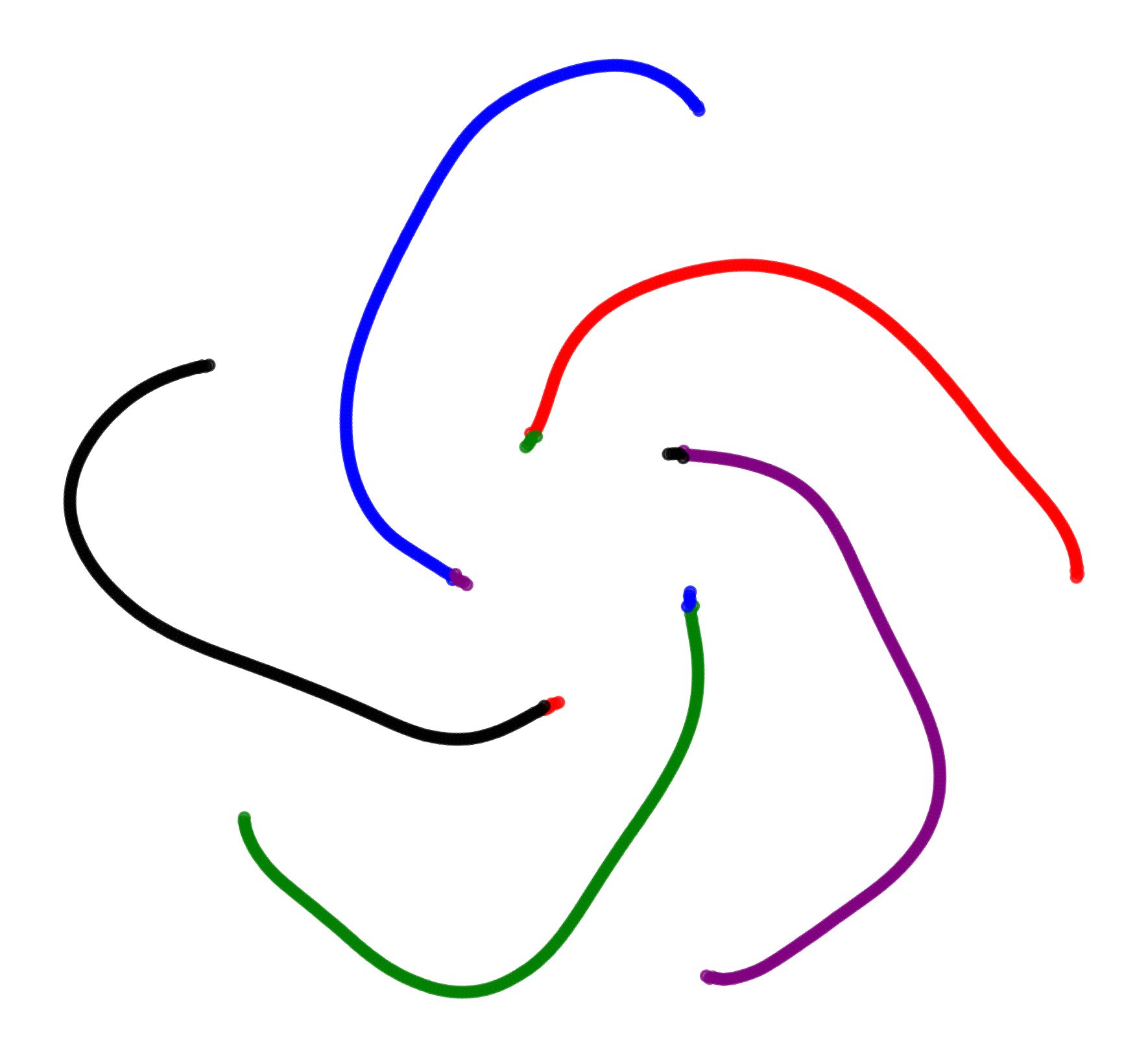}
        \caption{exag OFF, $\theta=0.8$}
    \end{subfigure}
    \caption{openTSNE embeddings from the same symmetric initialisation as
    Figure~\ref{fig:s1_vanilla}, under different early exaggeration and
    Barnes--Hut $\theta$ settings ($N=1000$, perplexity $10$). None is exactly
    symmetric; a quantification of the loss of symmetry is given in Table~\ref{tab:s1-combined}.}
    \label{fig:s1_sweep}
\end{figure}

\begin{table}[H]
  \centering
  \caption{Energies and symmetry defect for $\mathbb{S}^1$ with $T(x,y)=(x^2,y)$,
  $N=1000$,perplexity $10$. Vanilla gradient descent preserves the symmetry. Every openTSNE
  run does not exactly preserve the symmetry, and reaches a lower total energy.}
  \label{tab:s1-combined}
  \begin{tabular}{llcccc}
    \toprule
    Optimiser & Setting & $\mathrm{A_n}$ & $\mathrm{R_n}$ & $\mathrm{A_n}+\mathrm{R_n}$ & $\mathcal{SE}_{\text{final}}$ \\
    \midrule
    Vanilla GD & exact gradient
      & $0.5385$ & $-4.4705$ & $-3.9321$ & $3.088\times10^{-15}$ \\
    \midrule
    openTSNE & exag ON,\ $\theta=0$
      & $1.2580$ & $-5.5670$ & $-4.3090$ & $5.538\times10^{-2}$ \\
    openTSNE & exag ON,\ $\theta=0.8$
      & $1.1290$ & $-5.4230$ & $-4.2950$ & $1.182\times10^{-1}$ \\
    openTSNE & exag OFF,\ $\theta=0$
      & $1.3030$ & $-5.6150$ & $-4.3120$ & $2.115\times10^{-2}$ \\
    openTSNE & exag OFF,\ $\theta=0.8$
      & $1.0960$ & $-5.3850$ & $-4.2880$ & $3.757\times10^{-2}$ \\
    \bottomrule
  \end{tabular}
\end{table}

\begin{example}
Let $X = \Sbb^1$ denote the unit circle, with $\Gc_X$ generated by rotation by $\frac{2\pi}{9}$ 
and $\Gc_Y$ generated by rotation by $\frac{8\pi}{9}$. We use the function $T(x,y) = (x^2+\tfrac12 y,\, y^2)$ as an initialization
on the fundamental domain $\left[0, \frac{2\pi}{9}\right)$ and extend via symmetry. We run vanilla t-SNE on this setting with N=1800 points and perplexity=10 for 1000 iterations to obtain another perfectly symmetric plot with symmetry error as $1.483 \times 10^{-11}$. This is evidently a different map from example \ref{example:S1_cricle_rotation_2pi/5}. 
\end{example}

\begin{figure}[h!]
    \centering

    \begin{subfigure}{0.23\linewidth}
        \centering
        \includegraphics[width=\linewidth]{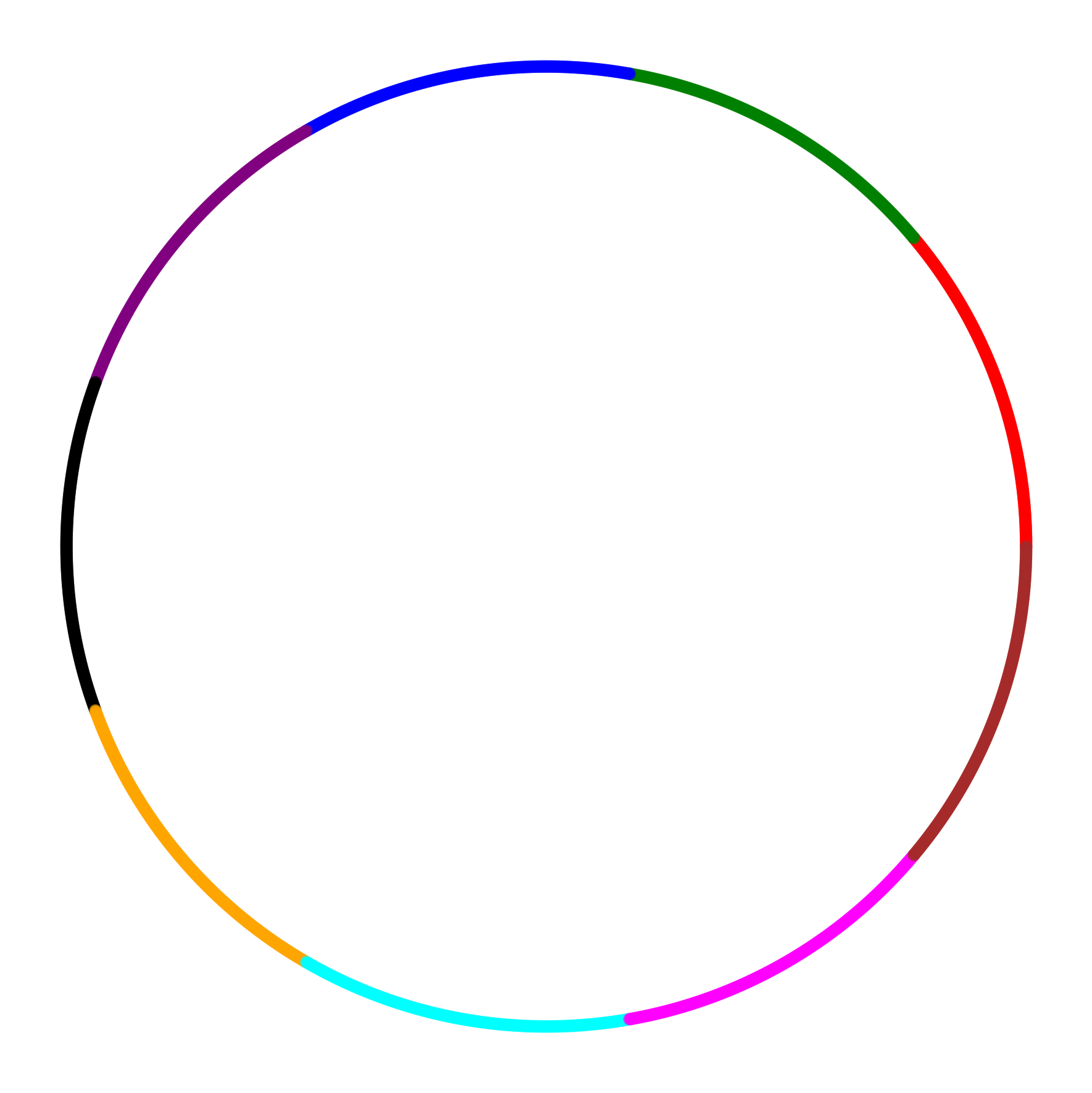}
        \caption{Points on the Circle}
    \end{subfigure}
    \hfill
    \begin{subfigure}{0.23\linewidth}
        \centering
        \includegraphics[width=\linewidth]{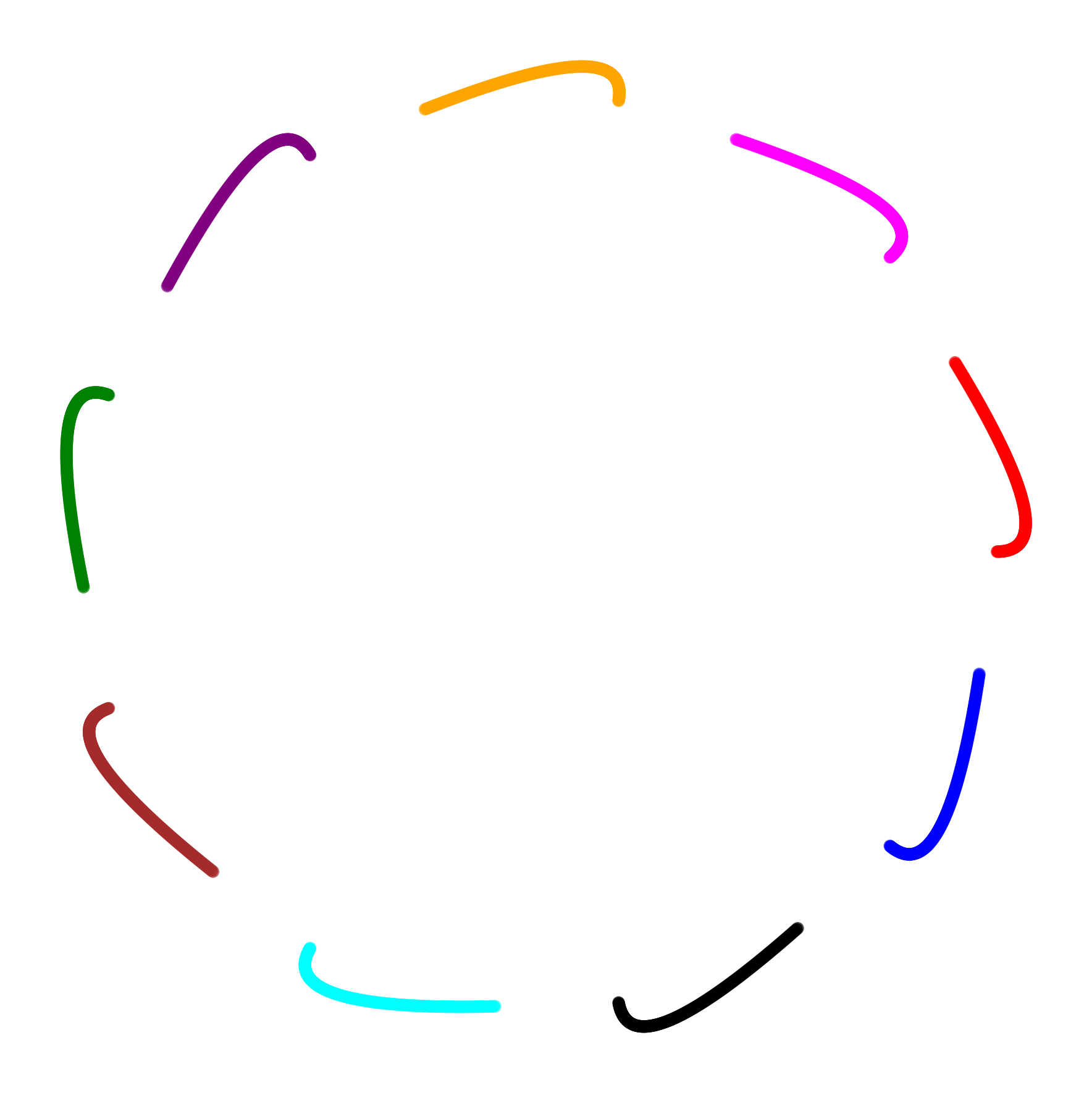}
        \caption{Initialization}
    \end{subfigure}
    \hfill
    \begin{subfigure}{0.23\linewidth}
        \centering
        \includegraphics[width=\linewidth]{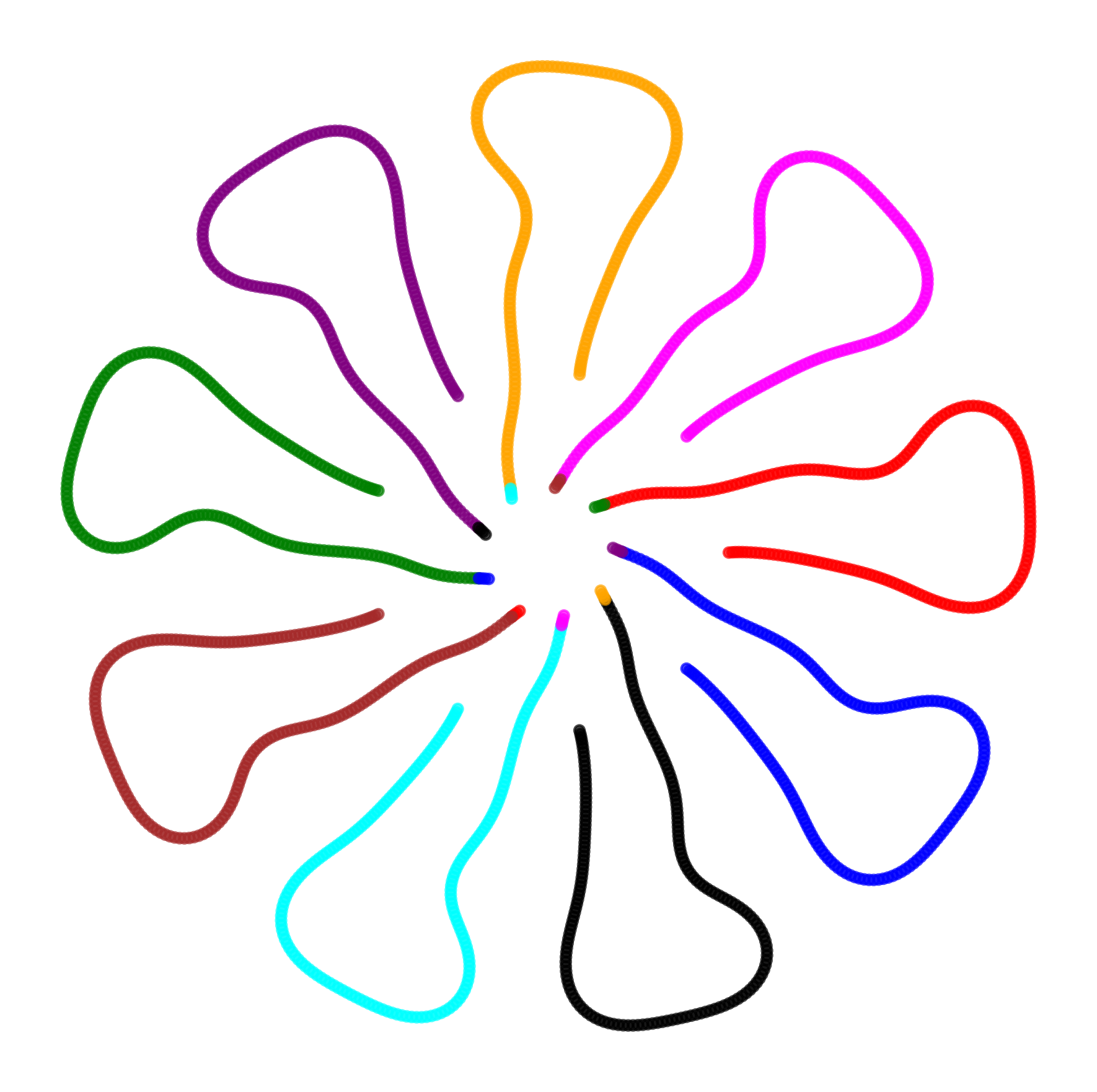}
        \caption{t-SNE}
    \end{subfigure}

    \caption{t-SNE plots for $\mathbb{S}^{1}$ circle}
\end{figure}

\begin{example}
The set $X$ is a square grid of points $(20 \times 20)$. Here $\Gc_X$ and $\Gc_Y$ are generated by 
 $g_X(x,y)=(x,-y)$, reflection across the x-axis, and $g_Y(x,y)=(-x,-y)$, rotation by $\pi$. We define $T$ on the fundamental domain 
\[
E = \{(x,y) \mid y \ge 0 \}
\]
to be the identity map, which determines $T$ on the whole domain as below,
\[
T(x,y) =
\begin{cases}
(x,y) & y \ge 0 \\
(-x,y) & y < 0
\end{cases}
\]
The numerical results, from the vanilla implementation of the t-SNE optimization with perplexity = 5, are shown in Figure \ref{fig:grid-twisting-sym}. These are very similar to the observed grid twisting which is exhibited in \citep{wattenberg2016use}, which we recreated in Example \ref{ex:grid_twisting}. Here the symmetry error for the opentSNE run is of the order of $10^{-2}$, and vanilla gradient descent produces a symmetry error of the order of $10^{-15}$.
\end{example}

\begin{figure}[H]
    \centering
    \begin{subfigure}{0.45\linewidth}
        \centering
        \includegraphics[height=4.5cm]{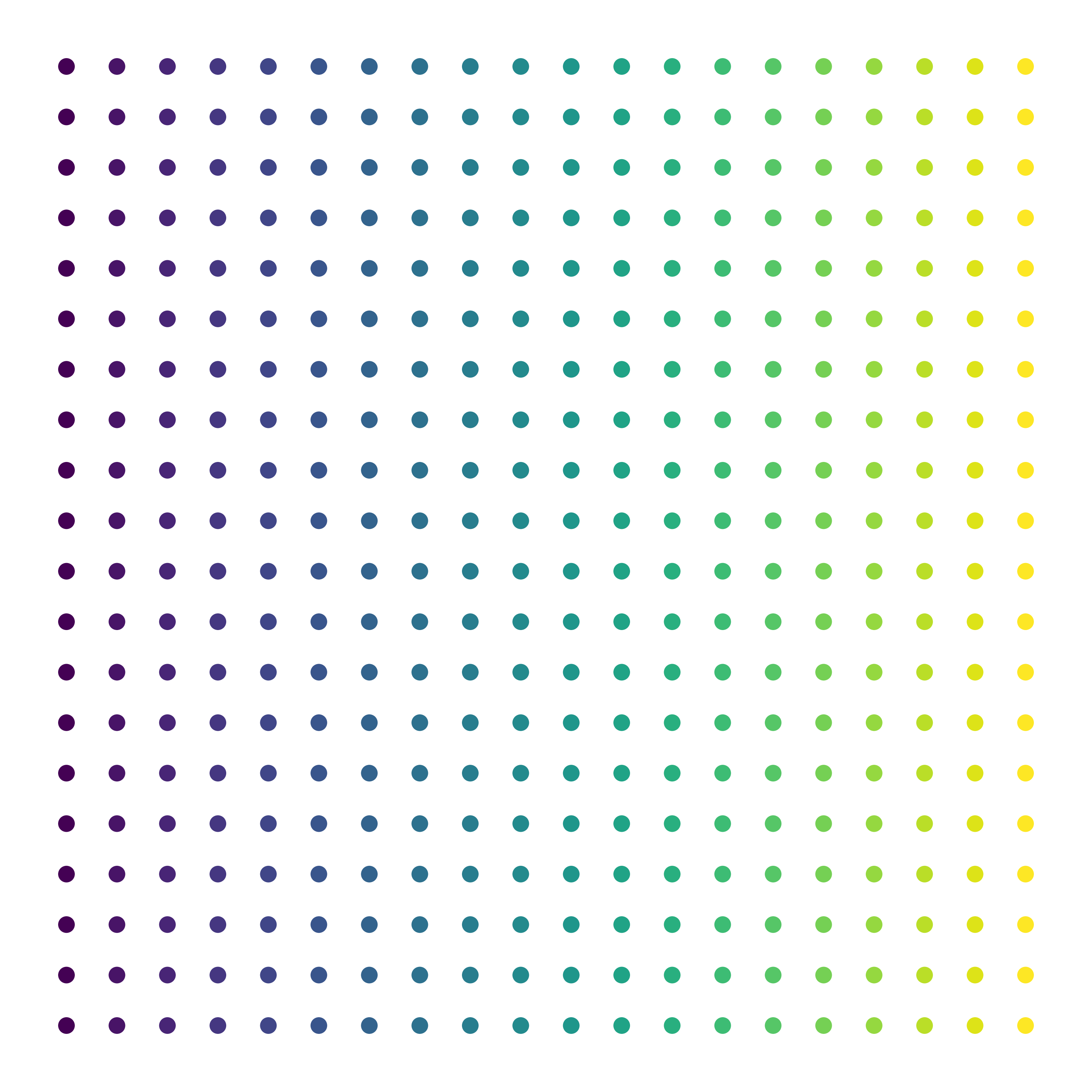}
        \caption{Original grid}
    \end{subfigure}
    \hfill
    \begin{subfigure}{0.45\linewidth}
        \centering
        \includegraphics[height=4.5cm]{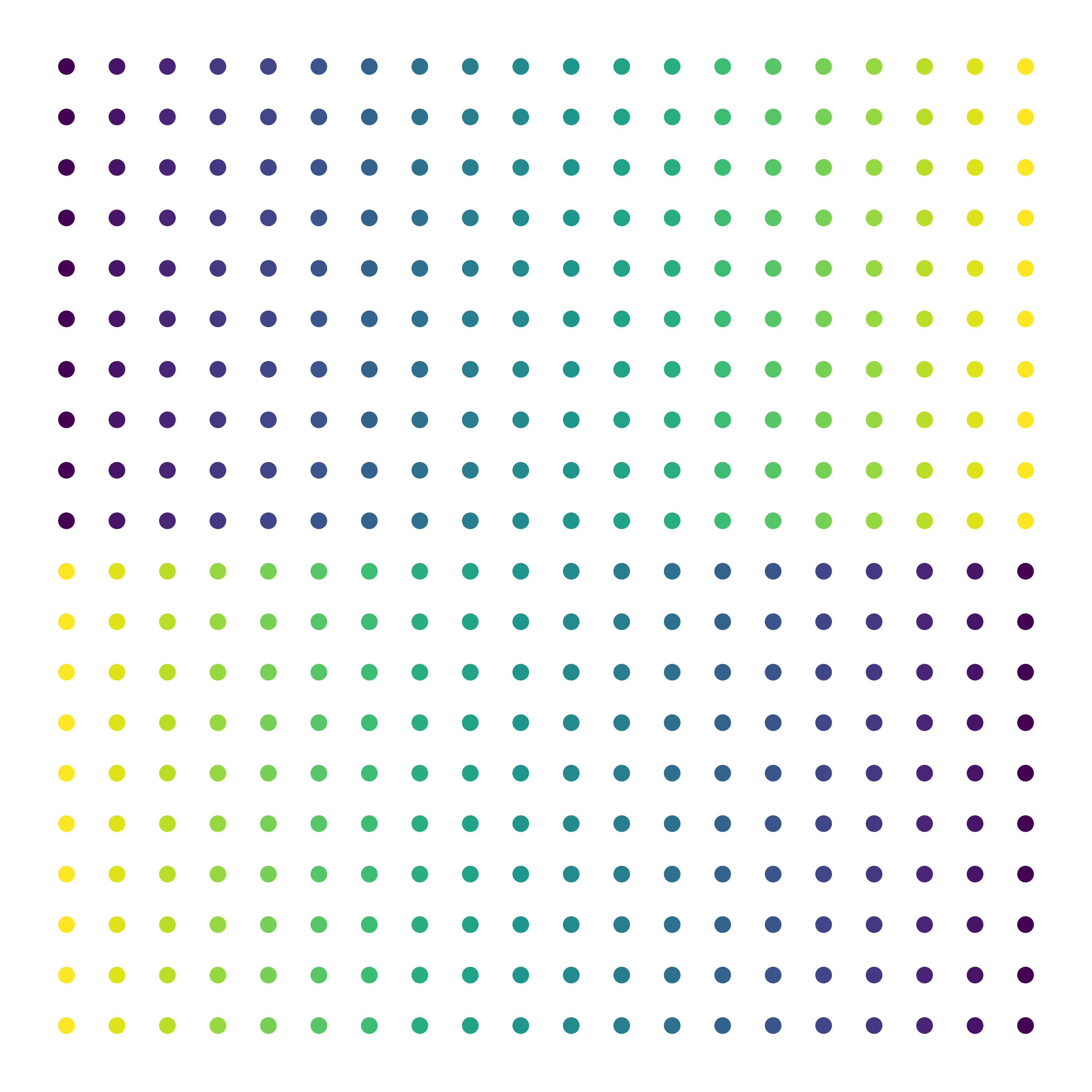}
        \caption{Initialization}
    \end{subfigure}

    \vspace{0.6em}

    \begin{subfigure}{0.45\linewidth}
        \centering
        \includegraphics[height=4.5cm]{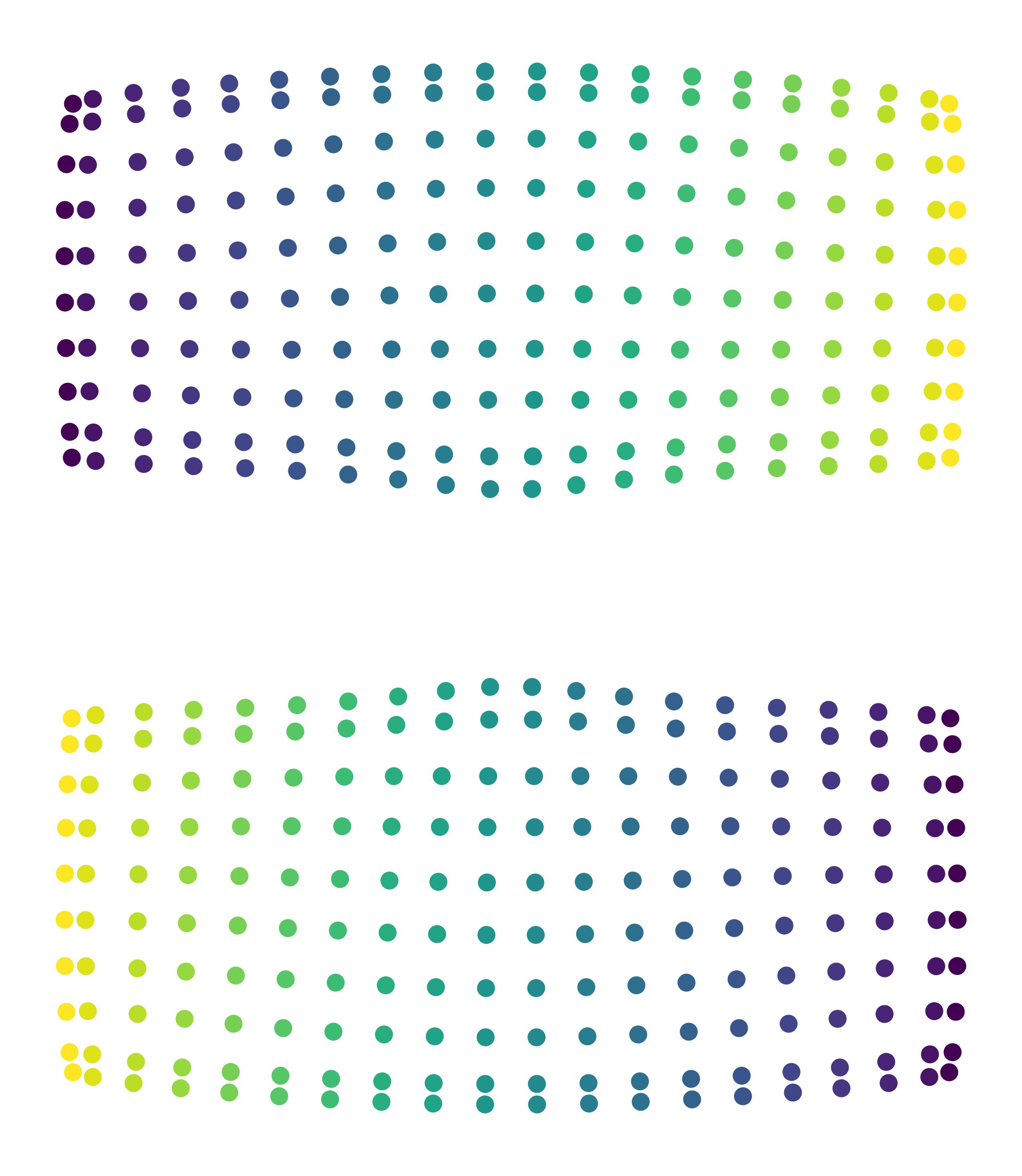}
        \caption{Vanilla t-SNE}
    \end{subfigure}
    \hfill
    \begin{subfigure}{0.45\linewidth}
        \centering
        \includegraphics[height=4.5cm]{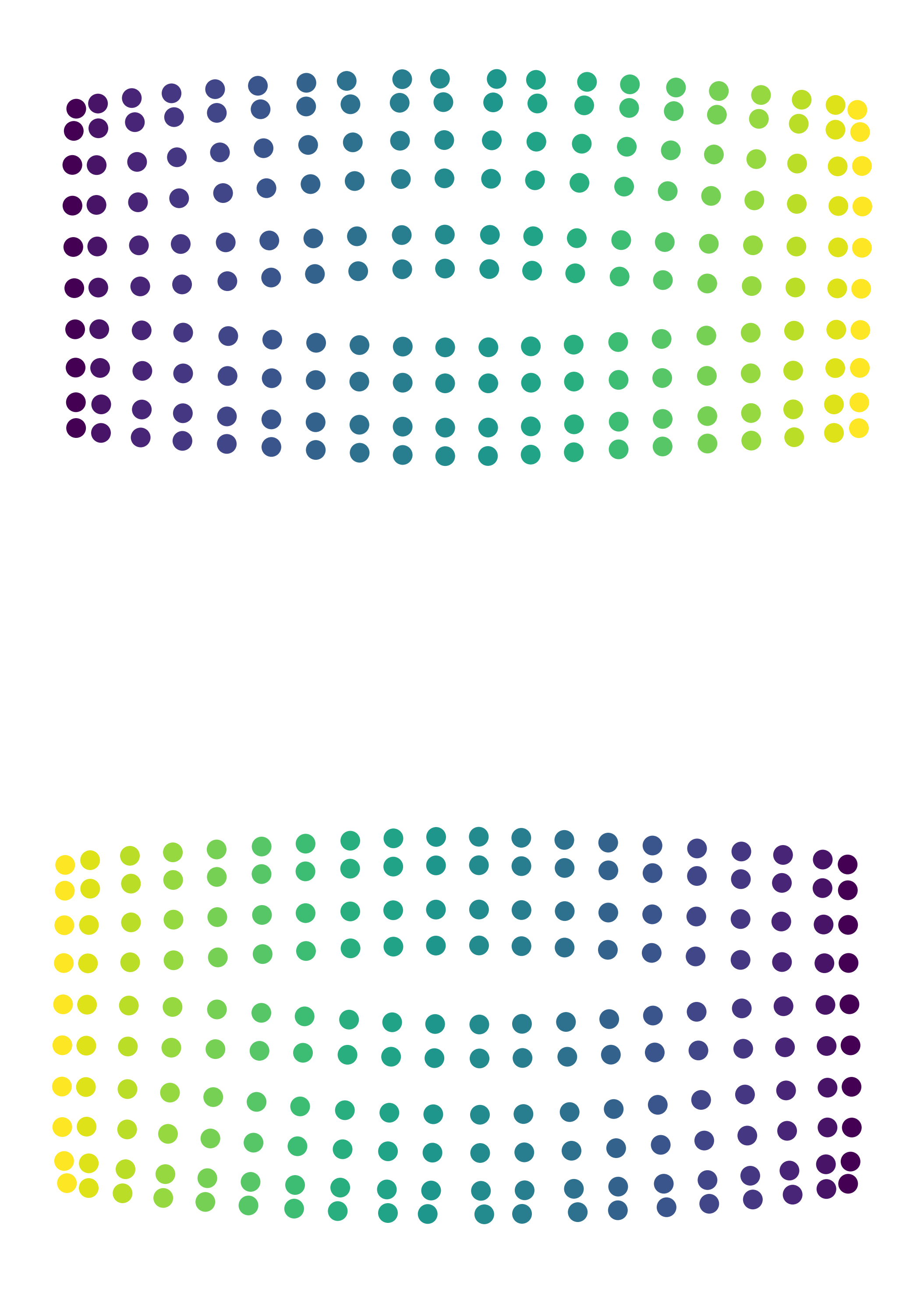}
        \caption{openTSNE}
    \end{subfigure}

    \caption{Grid twisting: the original grid, the symmetric initialization, and
    the embeddings produced by vanilla gradient descent and by openTSNE.}
    \label{fig:grid-twisting-sym}
\end{figure}


\begin{example}\label{ex:hyperbolic-shuffling}
   We consider a hyperbolic-space example using the t-SNE implementation of \citet{skrodzki2024hyperbolic}. Here we use $10000$ points on $\mathbb{S}^{2}$.The initialization is constructed by mapping each slice to a single point in hyperbolic space and here $g_Y$ is rotation by $\frac{4\pi}{5}$ The resulting embedding is computed in hyperbolic space and displayed using the Poincaré disk model. In the bottom row, we can see the evolution of the particles at various stages of the optimization. In the early stages of optimization the symmetry is preserved relatively well, but gets broken later on, probably due to the different optimization method used in this code formulation. 
\end{example}
\begin{figure}[H]
    \centering

    \begin{subfigure}{0.30\linewidth}
        \centering
        \includegraphics[width=\linewidth]{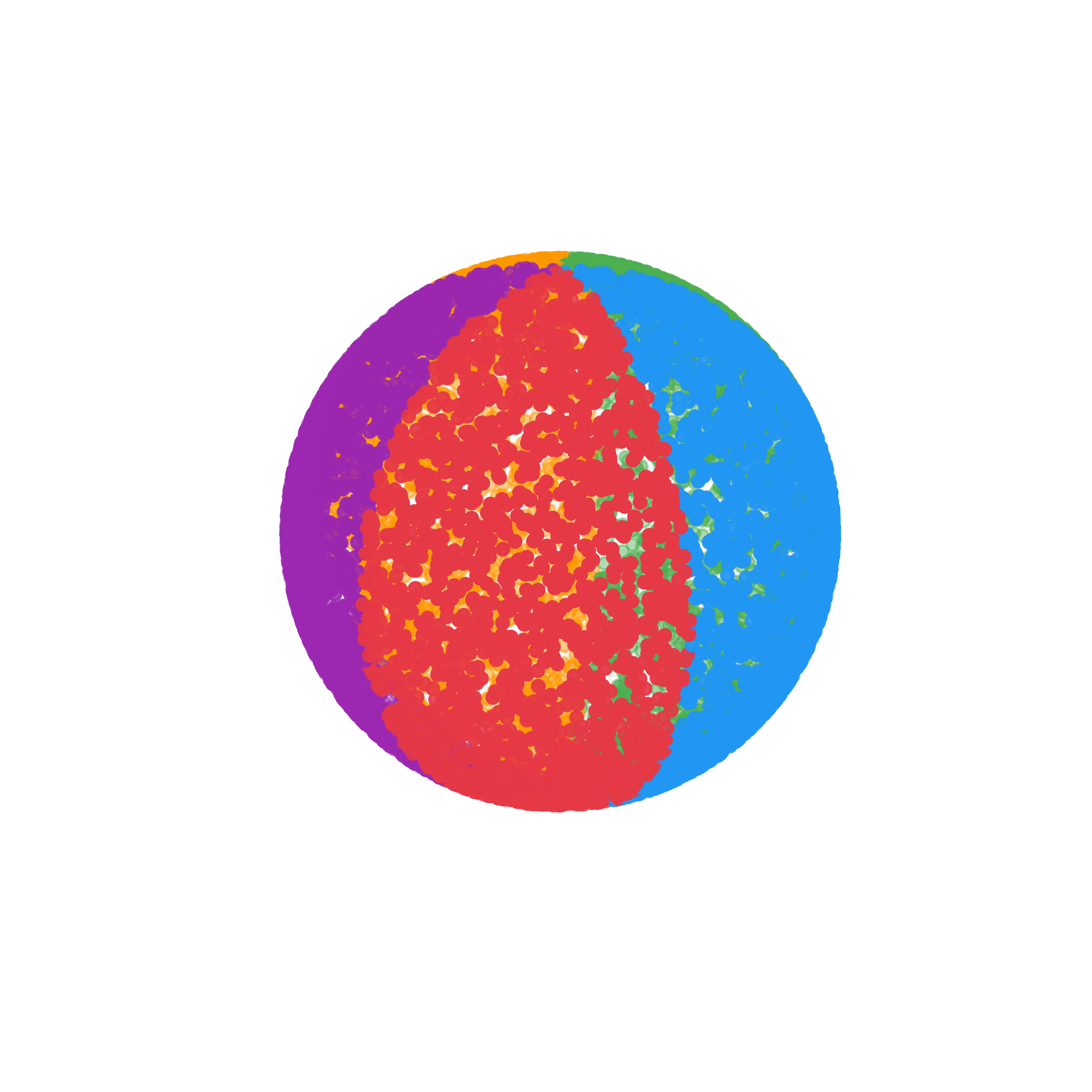}
        \caption{$S^{2}$}
    \end{subfigure}
    \hfill
    \begin{subfigure}{0.30\linewidth}
        \centering
        \includegraphics[width=\linewidth]{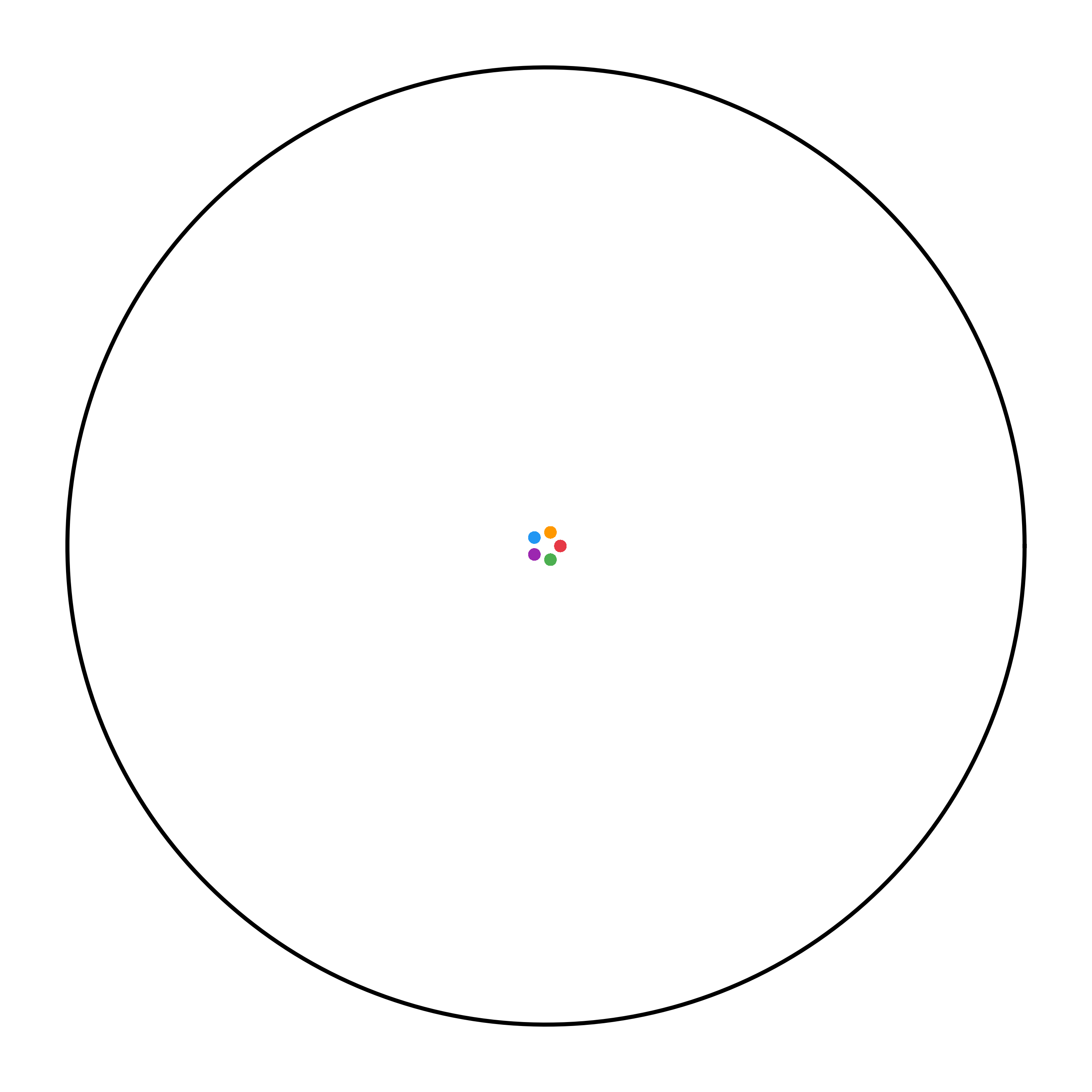}
        \caption{Initialization}
    \end{subfigure}
    \hfill
    \begin{subfigure}{0.30\linewidth}
        \centering
        \includegraphics[width=\linewidth]{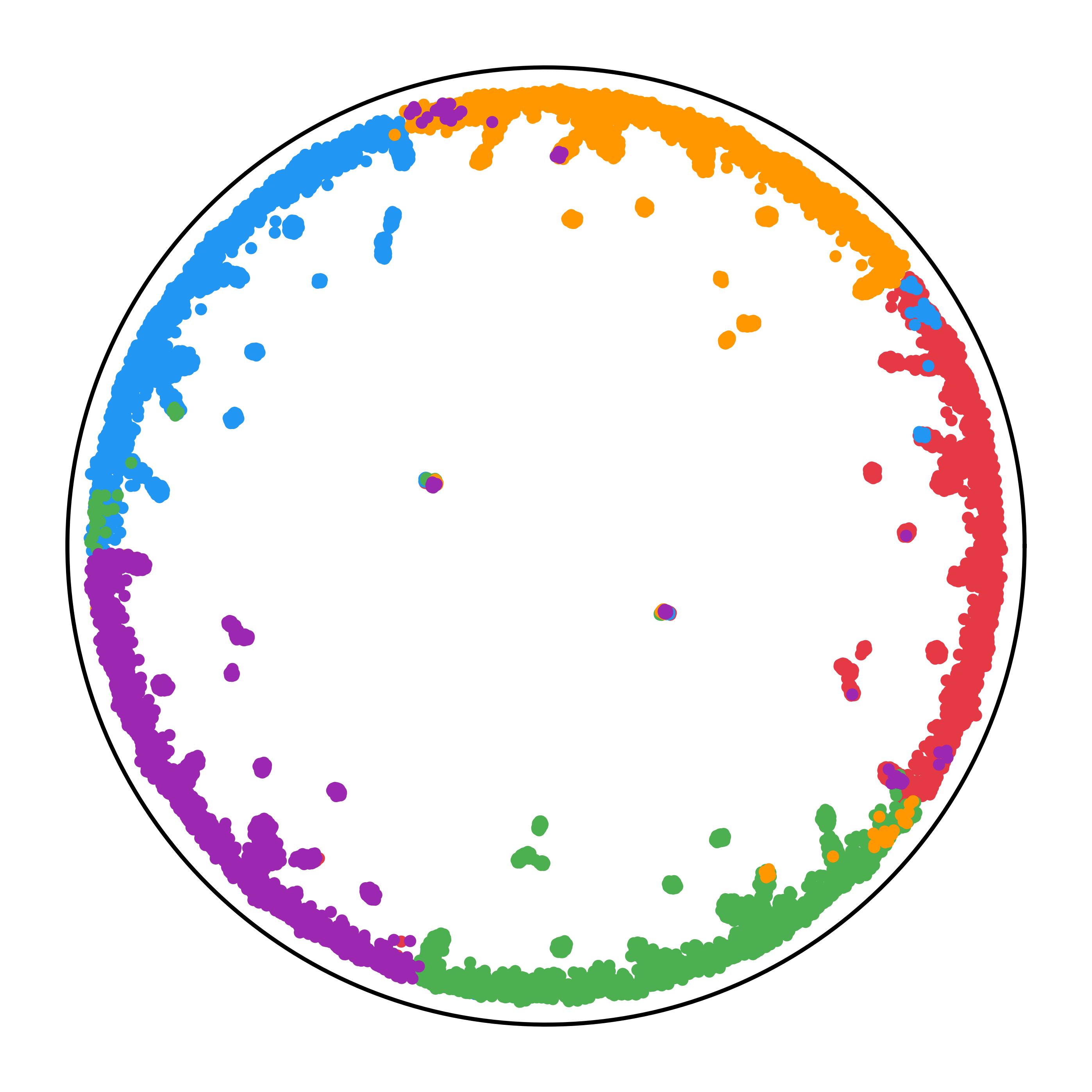}
        \caption{t-SNE output}
    \end{subfigure}

    \vspace{0.5cm}

    \begin{subfigure}{0.18\linewidth}
        \centering
        \includegraphics[width=\linewidth]{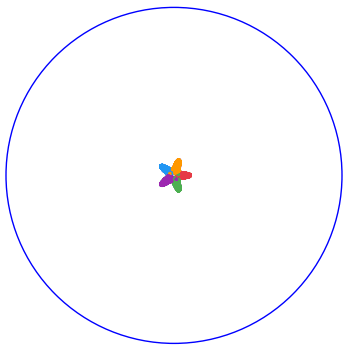}
        \caption{epoch:10}
    \end{subfigure}
    \hfill
    \begin{subfigure}{0.18\linewidth}
        \centering
        \includegraphics[width=\linewidth]{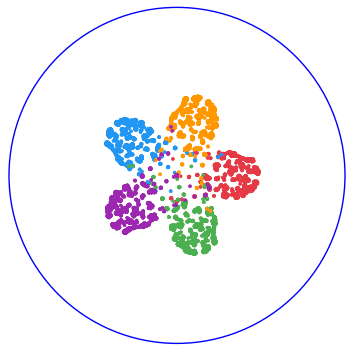}
        \caption{epoch:50}
    \end{subfigure}
    \hfill
    \begin{subfigure}{0.18\linewidth}
        \centering
        \includegraphics[width=\linewidth]{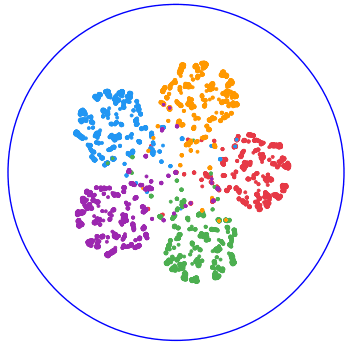}
        \caption{epoch:100}
    \end{subfigure}
    \hfill
    \begin{subfigure}{0.18\linewidth}
        \centering
        \includegraphics[width=\linewidth]{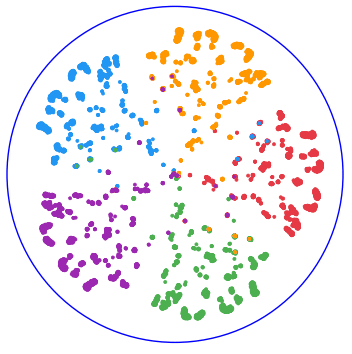}
        \caption{epoch:200}
    \end{subfigure}
    \hfill
    \begin{subfigure}{0.18\linewidth}
        \centering
        \includegraphics[width=\linewidth]{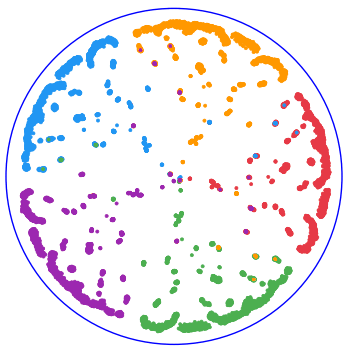}
        \caption{epoch:500}
    \end{subfigure}

    \caption{t-SNE plots for Sphere on Hyperbolic Space, described in Example \ref{ex:hyperbolic-shuffling}. 
    The bottom row shows the evolution of the particles during the optimization.}
    \label{fig:grid-twisting-sym}
\end{figure}

\section{Existence of descent paths and their limits}\label{sec:descent-paths}

In this section we address the issue of existence of descent paths, and their limits. In the finite data case, with $n$ points, one can re-frame the question of solving \eqref{eqn:GD-def} into an autonomous ordinary differential equation on $\R^{m \times n}$, which immediately gives local in time well-posedness. The question of a long-time limit in the finite data case reduces to understanding compactness properties of the solutions: this has been addressed in low-dimensional settings, specifically for $m=2$ in \citep{jeong2024convergenceanalysistsnegradient}. In the case of infinite data sets (e.g. with densities) these tools require some modification, which we describe here.

A first important step is to provide an a priori bound upon $I$. To begin, we recall a classical lower bound upon the energy, namely
\begin{equation}\label{eq:lower-bound-energy}
    \inf_{\pi} \mathcal{C}(\pi) = \mathcal{C}_{min} > -\infty.
\end{equation}
This inequality is a consequence of the positivity of the KL divergence: we omit the proof.
We then utilize this lower bound to provide a rough estimate on $\frac{1}{I}$, which we will use later on.

\begin{lemma}\label{lem:I-bound}
   Let $\pi_t$ evolve according to \eqref{eqn:GD-def}. Then for some $C>0$ only depending upon $C_\eta,\|\Xi\|_\infty,\mathcal{C}(\pi_0)$, we have \begin{equation}\label{eq:a-priori-I}
   \frac{1}{I(\pi_t)} \leq Ce^{Ct^{1/2}}.
   \end{equation}
\end{lemma} 
\begin{proof}
We notice first that
\[
\mathcal{C}(\pi_{t_2}) - \mathcal{C}(\pi_{t_1}) = -\int_{t_1}^{t_2}\int |\Ec[\pi_t](y)|^2 \pi_t(dx dy) dt.
\]
Equation \eqref{eq:lower-bound-energy}, then implies that
\[
\mathcal{C}(\pi_{t_0}) - \mathcal{C}_{min} \geq \int_{t_1}^{t_2}\int |\Ec[\pi_t](y)|^2 \pi_t(dx dy) dt =: \int_{t_1}^{t_2} H(t) dt .
\]
We then compute, for $Q(t) = \frac{1}{I[\pi_t]}$,
\begin{align*}
    \frac{d}{dt} Q &= \frac{1}{I^2} \int \frac{2\partial_1 \eta_Y(y,y')}{(1+\eta(y,y'))^2}\mathcal{E}[\pi_t](y)\pi_t(dx,dy) \pi_t(dx',dy') \\
    &= \frac{2}{I} \int \Ec[\pi_t](y) \left(\Ec[\pi_t](y) - \Xi(x,x')\frac{\partial_1\eta_Y(y,y')}{1 + \eta_Y(y,y')} \right) \pi_s(dx,dy) \pi_s(dx',dy') \\
    &\leq 2 Q(t) ( H(s) + CH(s)^{1/2}),
\end{align*}
where in the last line we used \ref{assump-Xi} as well as Jensen's inequality.

We note that by Cauchy-Schwarz we have 
\[
\int_0^t H(s)^{1/2} ds \leq \left(\int_0^t H(s) ds\right)^{1/2} \left(\int_0^t 1 ds\right)^{1/2} \leq (\mathcal{C}(\pi_0) - \mathcal{C}_{min})^{1/2} t^{1/2}.
\]
By then using Gr\"onwall's inequality, we obtain
\[
Q(t) \leq Q(0)\exp\left(\int_0^t H(s) + CH(s)^{1/2} ds\right) \leq CQ(0) e^{Ct^{1/2}}.
\]
We recall that $Q(0)$ will be finite for any $\pi_0$, by the finiteness of $\eta$. 
\end{proof}

\begin{lemma}\label{lem:I-lip}
Assume \ref{assump-symmetric_eta} and \ref{assump-eta_lipshitz}.
Let $\phi,\tilde \phi \in C([0,t_1),C^1(X\times Y;Y))$ and $\pi_t = (\Phi_t)_\sharp \pi_0, \tilde \pi_t = (\tilde \Phi_t)_\sharp \tilde \pi_0$. Then for any $t \in [0,t_1)$ there exists a $C$, so that
   \[
   \left| \frac{1}{I(\pi_t)} - \frac{1}{I(\tilde \pi_t)}\right| \leq \frac{C}{I(\pi_t) I(\tilde \pi_t)}  (\|\phi_t - \tilde \phi_t\|_\infty + Lip(\phi_t) W_1(\pi_0\otimes \pi_0,\tilde \pi_0 \otimes \tilde \pi_0)),
   \]
   where $W_1$ denotes the $1$-Wasserstein distance.
\end{lemma}

\begin{proof}
    We directly estimate
\begin{align}
&\left|\frac{1}{I(\pi_t)}-\frac{1}{I(\tilde \pi_t)}\right|  \leq \frac{|I(\tilde \pi_t)-I(\pi_t)|}{|I(\pi_t)I(\tilde \pi_t)|} \\
        &= \frac{1}{I(\pi_t) I(\tilde \pi_t)}\left|\int \frac{1}{1 + \eta(\phi_t(x,y),\phi_t(x',y')}d\pi_0(x,y) d\pi_0(x',y') - \int \frac{1}{1 + \eta(\tilde \phi_t(x,y),\tilde \phi_t(x',y')}d\tilde \pi_0(x,y) d\tilde \pi_0(x',y') \right| \\
        &\leq \frac{c_1}{I(\pi_t) I(\tilde \pi_t)}\|\phi_t - \tilde \phi_t\|_\infty + \frac{1}{I(\pi_t) I(\tilde \pi_t)}\left|\int \frac{1}{1 + \eta(\phi_t(x,y),\phi_t(x',y')} [d\pi_0(x,y) d\pi_0(x',y') - d\tilde \pi_0(x,y) d\tilde \pi_0(x',y')]\right| \\
        &\leq \frac{C}{I(\pi_t) I(\tilde \pi_t)}(\|\phi_t - \tilde \phi_t\|_\infty + Lip(\phi_t)W_1(\pi_0\otimes \pi_0,\tilde \pi_0 \otimes \tilde \pi_0)),
    \end{align}
    where at the second line we used the definition of the pushforward, in the third line we use the fact that $\frac{1}{1+\eta}$ is Lipschitz in its arguments (which follows from Assumptions \ref{assump-symmetric_eta} and \ref{assump-eta_lipshitz}), and in the last line we use the dual characterization of the $1$-Wasserstein distance along with the boundedness of $\frac{\partial \eta}{(1+\eta)^2}$ which follows from \ref{assump-eta_lipshitz}. 
\end{proof}

We note that in the case where $\pi_0 = \tilde \pi_0$ the estimate holds without needing $\phi_t$ to be Lipschitz, as the term disappears completely.


We now prove the existence of solutions to the evolution equation \eqref{eqn:GD-def}. The approach via diffeomorphisms is similar in spirit to the philosophy used for gradient descent methods in Wasserstein spaces \citep{ambrosio2005gradient}, but given the mixture of effects in the $X$ and $Y$ space, we cannot immediately apply those results. Instead, because the kernels which we use are relatively tame we prove well-posedness directly using the same proofs as in the classical differential equations setting. We remark that very recently equations of this type---in
which only one marginal is transported---have been studied as \emph{structured continuity equations} on
spaces of fibred probability measures
\citep{bonnetweill2025structuredcontinuityequationsfibred}: however, their general theory does not apply to our context because of the particular scaling of $I^{-1}$, which is comparable to second moments as opposed to first.

\begin{proposition}\label{prop:soln-exists}
Let $X$ and $Y$ be smooth manifolds and assume \ref{assump-eta_lipshitz} and \ref{assump-Xi}.
Then for any $\pi_0 \in \mathcal{P}(X\times Y)$ there exists a unique weak solution $\phi_t \in C\left([0,\infty),C^1(X \times Y,Y)\right)$ to \eqref{eqn:GD-def}. 
\end{proposition}

\begin{proof}

To begin, we focus our attention on the setting where $X,Y$ are Euclidean, which allows a simple integral formulation of weak solutions of ODEs. We follow the standard proof via the contraction mapping principle on the space $C\left([0,t_1),C^1(X \times Y,Y)\right) =: Z$. To this end, we define the Picard operator via the equation
\[
(\mathcal{T}\phi)(t,x,y)=y+\int_{0}^{t}\Ec[(\Phi_s)_\sharp \pi_0](x,\phi_s(x,y))ds,
\]
where $\Phi(t,x,y) = (x,\phi(t,x,y))$, as above. We notice that if $\mathcal T \phi = \phi$ then we have obtained a weak solution to \eqref{eqn:GD-def}.

We first notice that for any $\delta>0$ and for any $t_1$ finite, we have
\[
\inf_{\phi \in B_Z(Id_y,\delta), t \in [0,t_1]} I((\phi_t)_\sharp \pi_0) =: \underline{I} > 0. 
\]
This follows from the fact that $(1+\eta)^{-1}$ is locally bounded away from zero. By observing that all terms in the integrand (besides $I^{-1}$) of $\Ec$ are continuous and bounded then implies that $\Ec[(\Phi_s)_\sharp \pi_0] \in C^1(X\times Y ; Y)$, and hence for $t_1$ sufficiently small $\mathcal T$ maps $B_Z(Id_y,\delta)$ to itself.

What remains then is to estimate the difference
\[
(\mathcal{T}\phi)(t,x,y)-(\mathcal{T}\psi)(t,x,y)=\int_{0}^{t}\Ec[(\Phi_s)_\sharp \pi_0](x,\phi_s(x,y))-\Ec[(\Psi_s)_\sharp \pi_0](x,\psi_s(x,y)) ds.
\]
We recall that
\begin{align*}
&\Ec[(\Phi_s)_\sharp \pi_0](x,\phi_s(x,y))\\
&=\int_{X \times Y} 2\Xi(x,x')\frac{\partial_{1}\eta_Y(\phi_s(x,y),\phi_s(x',y'))}{(1+\eta_Y(\phi_s(x,y),\phi_s(x',y'))) } + \frac{-2}{I((\Phi_s)_\sharp \pi_0)} \frac{\partial_{1}\eta_Y(\phi_s(x,y),\phi_s(x',y'))}{(1+\eta_Y(\phi_s(x,y),\phi_s(x',y')))^2 }  \pi_0(dx'dy').
\end{align*}
By the boundedness of $\Xi$, Lemma \ref{lem:I-lip}, the Lipschitzness of $\frac{\partial_1 \eta}{1+\eta}$ and $\frac{\partial_1 \eta}{(1+\eta)^2}$, we then obtain, with $C$ independent of $t$
\[
\|\mathcal{T}\phi - \mathcal{T}\psi\|_Z \leq \frac{C t_1}{\underline I^2}\|\phi-\psi\|_Z.
\]
By then making $t_1$ small enough, we obtain a contraction mapping, and hence a unique fixed point, which corresponds to a unique weak solution to \eqref{eqn:GD-def}. Using the a priori bound on $I^{-1}$ from Lemma \ref{lem:I-bound} we can then uniquely extend this solution for all finite times.



The case where $X$ and $Y$ are manifolds is completely analogous, one rewrites everything using charts and carry out similar computations as the Euclidean case. For details in a similar setting, we refer to Theorem A.3 from \citep{fetecau2021wellposednessasymptoticbehaviouraggregation}.
\end{proof}

Having demonstrated the existence of a (symmetry preserving) gradient descent path, we now turn our attention to the question of limits of those descent paths. We begin with a standard lemma, which is an immediate consequence of Prokhorov's theorem.

\begin{lemma}\label{lem:weak-conv}
    Let $\pi_t$ evolve according to \eqref{eqn:GD-def}. Suppose that either i) $Y$ is compact or ii) $Y=\R^m$ and $\limsup_{t \to \infty} \mathbb{E}_{\pi_t}(|y|)< \infty$. Then (up to a subsequence of times) $\pi_t$ converges weakly to a limit $\pi_\infty$.
\end{lemma}

In the Euclidean case, the assumption that the moment in $Y$ remain bounded is currently justified in some contexts, namely for finitely many point masses and $Y = \R^2$ \citep{jeong2024convergenceanalysistsnegradient}. Motivated by those results, we give the following conjecture: 

\begin{conjecture}\label{conj:moment-bound}
    For $X = \R^d$, and $Y=\R^m$, with an initial coupling $\pi_0$ that has finite first moment, then the first moment remains uniformly bounded for all times.
\end{conjecture}

\begin{remark}
It is interesting to consider what would happen if the moments failed to be bounded in time. Letting $M_t = \mathbb{E}_{\pi_t}(|y|)$, we notice that by defining $\tilde \pi_t (A\times B) := \pi_t(A \times M_t B)$, we obtain, after some algebra,
\[
\mathcal{C}( \pi_t) = \log \iint \frac{1}{M_t^{-2} + |\tilde y - \tilde y '|^2}d\tilde \pi_t d\tilde \pi_t' + \iint \Xi(x,x') \log(M_t^{-2}+|\tilde y -\tilde y'|^2) d\tilde \pi_t d\tilde \pi_t'.
\]
It is then natural to then consider the zero homogeneous energy (ZHE)
\[
ZHE(\tilde \pi) := \log \iint \frac{1}{ |\tilde y - \tilde y '|^2}d\tilde \pi d\tilde \pi' + \iint \Xi(x,x') \log(|\tilde y -\tilde y'|^2) d\tilde \pi d\tilde \pi'.
\]
This energy is now scale invariant, in the sense that global scalings in $Y$ do not affect the energy. Such an energy was previously alluded to in \citep{calder2026continuumlimittsnedata}. While it is tempting to guess that the weak limit $\tilde \pi_\infty$ of the $\tilde \pi_t$ will be a critical point of $ZHE$, this does not immediately follow as
\[
\Ec(\pi_t) \sim \frac{\Ec_\infty(\tilde \pi_t)}{M_t},
\]
where here $\Ec_\infty$ is the zero homogeneous analog of $\Ec$ (namely one removes all of the ``$+1$'' terms in the integrands), and we remember that we are in this remark considering the regime where $M_t \to \infty$.

However, by the lower bound on the energy $\mathcal{C}$ we know that the $L^2$ norm of $\mathcal{E}(\pi_t)$ has a finite integral on $[0,\infty)$. We also note that the proof of the a priori bounds upon $I$ suggests that $M_t$ may grow algebraically slowly, which could still give implications about the $L^2$ magnitude of $\Ec_\infty(\tilde \pi_t)$. However, lacking more precise bounds upon $M_t$, one cannot conclude that $\tilde \pi_\infty$ is a critical point of $ZHE$. 
\end{remark}

We then give one of the main results of this work, regarding the existence of symmetry invariant critical points.

\begin{theorem}\label{thm:crit-exist}
	Suppose that  \ref{assump-symmetric_eta},\ref{assump-eta_lipshitz},\ref{assump-Xi},(\ref{assump:sigma_invariance}) hold and that either 1) $Y$ is compact or 2) $Y=\R^m$ and Conjecture \ref{conj:moment-bound} holds. Then for every symmetry pair $(\Gc_X,\Gc_Y)$ for which there admits at least one symmetry invariant $\pi_0$ there exists a symmetry invariant $\pi^*$ which is a  critical point of the energy \eqref{eqn:total-energy}.
	\end{theorem}
	
	\begin{proof}
		By Proposition \ref{prop:soln-exists}, and Theorem \ref{thm:sym-invar-preserved}, we have a symmetry invariant family $\pi_t$ satisfying the energy relation
		\[
		\mathcal{C}(\pi_{t_2}) - \mathcal{C}(\pi_{t_1}) = -\int_{t_1}^{t_2}\int |\Ec[\pi_t](y)|^2 \pi_t(dx dy) dt.
		\]
	As $\mathcal{C}$ is lower bounded (see Equation \ref{eq:lower-bound-energy}), we then have that $\Ec[\pi_{t_n}] \to 0$ for some $t_n \to \infty$. Using Lemma \ref{lem:weak-conv} (Weak convergence) there exists (up to a subsequence) a weak limit $\pi_\infty$ of this sequence, which will also be symmetry invariant. Finally, using the continuity of the integrand in $\Ec$, we then obtain that $\Ec[\pi_\infty] = 0$, as desired. 		
		\end{proof}

We do expect that one should be able to show that a minimizer exists within the family of symmetry invariant plans, depending on the method of proof for Conjecture 9.

This main theorem allows the construction of a multitude of critical points for the t-SNE energy. Without additional assumptions it is not possible to say whether such critical points are stable or local minimizers, but the numerical examples in Section \ref{sec:num-examples} indicate that, at least for some examples, these symmetries are relatively stable.

\section{Non-minimality of trivial maps in Euclidean settings}\label{sec:distinctness}

In previous sections we have demonstrated that gradient descent dynamics for t-SNE are well-defined and preserve certain isometry pairs. Furthermore, we have shown that by following these descent paths, we must approach symmetry invariant critical points of the energy.

One of the central goals of this work is to rigorously prove that the energy landscape of t-SNE admits \emph{many} critical points. In contexts where the symmetry pairs are \emph{mutually exclusive} (see Section \ref{sec:symmetries} and Example \ref{ex:torus-mutually-exclusive}), one can readily infer that either the critical points associated with the symmetry pairs are distinct or they have approached a symmetry invariant correspondence (an outcome we view as unlikely for finite bandwidths, see Conjecture \ref{conj:maps}).

However, in many contexts, especially the natural Euclidean setting with rotational symmetries, the symmetry $g_Y$ admits at least one fixed point, and hence one can construct many symmetry pairs which admit the same fixed point. In Section \ref{sec:symmetries} we developed terminology for this by saying that two symmetry pairs could be \emph{mutually exclusive up to constants}. In such contexts, arguing that the limits of the t-SNE gradient descent are distinct then simply requires that the t-SNE energy not be locally minimized by constants within the family of symmetry invariant maps. 

Pursuing this program requires computations which are specific to the energy and manifolds in question. As a proof of concept, we show it in the Euclidean setting, and then make comments about the parts that would need to be modified to adapt to other manifolds. To begin, we first give two elementary lemmas, regarding the $k$-periodic structure of $Dg_Y$ at stationary points.

\begin{lemma}\label{lem:lin-isometry-k-period}
    Let $Y$ be a smooth, connected, complete  Riemannian manifold. Suppose that $\Gc_Y$ is a finite, cyclic group of isometries on $Y$ with $k$ elements, where $k$ is prime, generated by $g_Y$, and suppose that $g_Y(y) = y$ for some $y \in Y$. Then $Dg_Y(y)^j \neq I$ for $1<j<k$ but $Dg_Y(y)^k  = I$.
\end{lemma}

\begin{proof}
    The fact that $Dg_Y(y)^k = I$ follows by using the fact that $g_Y^k(\tilde y) = \tilde y$ for any $\tilde y$, using a Taylor expansion of $g_Y$ and then taking a limit as $\tilde y \to y$.

    To show that $Dg_Y(y)^j \neq I$, we first describe a natural property of length minimizing paths. We call the period of a point $\tilde y$ the number of elements in its orbit under $\Gc_Y$. Now, given any point $\tilde y$ with orbit of period $\tilde k$, we let $\gamma: [0,1] \to Y$ be a length minimizing path between $y$ and $\tilde y$. As $g_Y$ is an isometry, we also have $g_Y^j(\gamma)$ is a length minimizing path between $y$ and $g_Y^j(\tilde y)$. As the endpoints of these geodesic paths are distinct for $j = 1 \dots \tilde k-1$, then their initial velocities are also distinct, namely $Dg_Y(y)^j\gamma'(0)$ is distinct for $j = 1 \dots \tilde k-1$.

    Now, we note that by assumption, $\Gc_Y$ has $k$ distinct elements, and hence for any $j<k$ there must exist a $\tilde y$ so that $g_Y^j(\tilde y) \neq \tilde y$: this readily implies that there exist points which have period greater than one. We also notice that the period $\tilde k$  of any point $\tilde y$ must divide $k$. In the case when $k$ is prime, this actually directly implies that there must exist a point with orbit of length $k$: the previous paragraph would then imply the desired result.

    
    
\end{proof}
We note that $Dg_Y(y)$ is always an orthogonal transformation as $g_Y$ is an isometry. This fact then enables the following lemma, which is linear algebraic in nature. 
\begin{lemma}\label{lem:periodic-rotation-cancel}
    Let $U$ be an orthogonal transformation operator on $Tan_Y(y)$. Suppose that $U^j \neq Id$ for $1 \leq j < k$ and $U^k = Id$, with $k$ prime. Then there exists a non-zero vector $v \in Tan_Y(y)$ such that $\sum_{i=1}^k U^i v = 0$.
\end{lemma}

\begin{proof}
We can write a matrix representation of $U$ as a real orthogonal matrix. Such a matrix admits a canonical representation (after an orthogonal change of basis) into a block diagonal form with the blocks being $2\times 2$ rotation matrices. At least one of those blocks will be idempotent of order $k$, due to the primeness of $k$. A unit vector $v$ in the two-dimensional subspace associated with such a block will rotate in that subspace with period $k$, and will satisfy the conclusion of the lemma.

\end{proof}

We now establish our energetic result in the Euclidean setting, namely that constant maps are not local minimizers (even within the class of symmetry invariant maps).

\begin{proposition}\label{prop:trivial-non-optimal}
    Let $X=\mathbb{R}^{d}$ and $Y=\mathbb{R}^{m}$, and $\Gc_X,\Gc_Y$ be finite isometry groups with $k$ elements, where $k$ is prime, generated by $g_X$ and $g_Y$ respectively, and let the only fixed point under $\Gc_Y$ be $y_0$.


    Assume that $\mu_X(E_k)=1$, where $E_k$ is the set of all points with orbit size equal to k.
    Then the trivial map $T_0(x) \equiv y_0$ is not a local minimum of the t-SNE energy \eqref{eqn:total-energy} within the class of $(g_X,g_Y)$ invariant mappings.
\end{proposition}

\begin{proof}
One can directly verify that the map $T_0$ is a critical point of the t-SNE
energy. We will construct a perturbation which will decrease the t-SNE energy,
i.e.\ have a negative second variation. Let
$\pi_0(dx,dy)=\mu_X(dx)\delta_{y_0}(dy)$.

We know that $\Xi(x,x')=\frac{1}{2}\big(\gamma(x,x')+\gamma(x',x)\big)$. 
Let $ c(x)= \int_X \gamma(x',x)\,\mu_X(dx')$ we have
\[
s(x) \;:=\; \int_X \Xi(x,x')\,\mu_X(dx')
\;=\; \frac{1+ c(x)}{2},
\]
Since $\int_X s(x)\,d\mu_X = 1$. As $s$ is
continuous and positive with mean one, there exist $x^*$ and $\epsilon>0$ such
that $s \leq 1$ on $B(x^*,\epsilon)$. Here by assumption
take $x^*$ has orbit length $k$. As
$\mu_X$, $\sigma$ and $\eta_X$ are $g_X$ invariant, so is $s$, and hence
$s \leq 1$ on $g_X^i(B(x^*,\epsilon))$ for every $i$.

Using Lemmas \ref{lem:lin-isometry-k-period} and
\ref{lem:periodic-rotation-cancel}, select nonzero $v$ so that
$\sum_{i=1}^k Dg_Y(y_0)^i v = 0$. Select $0 < r < \epsilon$ small enough that
the sets $g_X^i(B(x^*,r))$ are disjoint and define
\[
\tau(x) = \sum_{i=1}^k \chi_{g_X^i(B(x^*,r))}(x)\, Dg_Y(y_0)^i v ,
\qquad \tau(x,y_0) := \tau(x).
\]
By the $g_X$ invariance of $\mu_X$ and $\sigma$ we have that $\int \tau(x) \mu_X(dx)=0$.

Computing the second variation of the energy gives
\begin{align*}
    \delta^2\mathcal{A}+\delta^2\mathcal{R}
    &= -\iint |\tau(x,y)-\tau(x',y')|^2\,d\pi_0 d\pi_0'
     + \iint \Xi(x,x')\,|\tau(x,y)-\tau(x',y')|^2\,d\pi_0 d\pi_0' \\
    &= 2\int\big(s(x)-1\big)|\tau(x)|^2\,d\mu_X
     - 2\iint \Xi(x,x')\langle\tau,\tau'\rangle\,d\mu_X d\mu_X
     + 2\Big|\int \tau\,d\mu_X\Big|^2 .
\end{align*}
The last term vanishes by the choice of $\tau$ above, and since $s \leq 1$ on
the support of $\tau$, which is zero elsewhere, the first term is non-positive.
Hence
\[
\delta^2\mathcal{A}+\delta^2\mathcal{R}
 \;\leq\; -2\iint \Xi(x,x')\,\langle\tau,\tau'\rangle\,d\mu_X\,d\mu_X .
\]
By the positive definiteness of the Gaussian kernel,  so that for
$r$ sufficiently small, using the continuity of $\gamma$,


\[
\iint \Xi(x,x')\,\langle\tau(x),\tau(x')\rangle\,d\mu_X\,d\mu_X \;>\; 0 ,
\]
and therefore $\delta^2\mathcal{A}+\delta^2\mathcal{R} < 0$.

Finally, by construction, we have that $\tau$ satisfies the invariance condition \eqref{eq:time-deriv-commutator}: 
\[
\tau(g_X(x)) = Dg_Y(y_0)\,\tau(x),
\]
implying that we can construct a vector field preserving the symmetry with this initial velocity, completing the proof.
\end{proof}

 In the previous proof much of the construction was applicable for any general $Y$ (and we retained much of the manifold notation intentionally!). One point which changes is in the computation of the second variation, but the same formula holds if $\xi(0) = 0, \xi'(0) = 1,$ and $\xi''(0) = 1$. Another point that changes in the Riemannian setting is the point when we invoked the positive definiteness of the Gaussian kernel. This is generally known to fail in Riemannian settings \citep{jayasumana2014kernelmethodsriemannianmanifold}, although appropriate generalizations exist.
 However, we note that if there exists a fixed point $x_0$ of $g_X$ at which $\mu_X$ is supported then one can select $x^*$ sufficiently close to $x_0$ so that the kernel remains positive definite, and in that setting the proof should remain valid.
 
 Finally, as illustrated in Example \ref{ex:correspondence-symmetries}, the mutual exclusivity of two symmetry pairs often depends upon restricting to the class of \emph{maps} as opposed to \emph{plans}. We note that embedding plans are problematic from the practical standpoint: if an embedding algorithm genuinely maps the same features to multiple embedded locations then interpreting the embeddings becomes challenging. From that standpoint, we make the following conjecture:
 \begin{conjecture}\label{conj:maps}
 Suppose that $\pi_0$ is a $(\Gc_X,\Gc_Y)$ symmetry-invariant plan which is induced by a measurable map. Then the weak limit of $\pi_t$ as $t \to \infty$ is also induced by a measurable map.
 	\end{conjecture}
 	A positive or negative resolution of this conjecture would be quite interesting, both from a mathematical and practical perspective. A simplified version, in the case where $\Gc_X,\Gc_Y$ are both the trivial groups, addresses the question in the context of unconstrained optimization. We note that a positive version of this conjecture was obtained for a simplified dimension reduction algorithm in \citep{murray2025probabilistic}. In the case of the energy from this work we still have that the attraction energy penalizes multi-valuedness, while the repulsion energy only observes the marginal in $Y$.
 	
 	In light of this conjecture, we give a final, conditional result:
 	
 	\begin{theorem}\label{thm:abundant_critical_points}
 		Let $X = \R^d$, $d > 1$ and $Y=\R^m$, $m\geq 2$. Suppose that assumptions \ref{assump-symmetric_eta},\ref{assump-eta_lipshitz},\ref{assump-Xi},(\ref{assump:sigma_invariance}) and that Conjectures \ref{conj:moment-bound} and \ref{conj:maps} hold. Then for a radially symmetric $\mu_X$ we have that there exist infinitely many distinct critical points of the t-SNE energy. 
 		\end{theorem}
 		
 		\begin{proof}
 			By Theorem \ref{thm:crit-exist}, for every symmetry pair (for which there is at least one symmetry invariant coupling) there is a symmetry invariant critical point. Conjecture \ref{conj:maps} implies that this critical point must be induced by a map.

            Now, for radially symmetric $\mu_X$ we can select $\Gc_X$ and $\Gc_Y$ to be finite groups of rotoreflections. As evidenced by Example \ref{ex:rotoreflections}, we can construct infinitely many such groups which are mutually exclusive up to constants. By Proposition \ref{prop:trivial-non-optimal} the trivial (i.e. constant) map is not locally optimal among symmetry invariant embeddings. Hence the critical point we've constructed which is associated with each symmetry pair is distinct, which proves the theorem.
 		\end{proof}

\acks{The authors have received partial support from NSF DMS 2307971 and Simons MP-TSM. The authors also thank Adam Pickarski for helpful discussion and comments on early versions of this work.}

\bibliography{references}
\end{document}